%% file: main.tex
\documentclass{article}

\usepackage{algorithmic}
\usepackage{algorithm}

\usepackage{graphicx}
\usepackage{subfigure}

\usepackage{multirow}  
\usepackage{array}  
\usepackage{makecell} 
\usepackage{tablefootnote}

\usepackage{amsmath}

\usepackage{subcaption}

\usepackage[final]{neurips_2025}

\usepackage{enumitem} 
\usepackage[utf8]{inputenc} 
\usepackage[T1]{fontenc}    
\usepackage{hyperref}       
\usepackage{url}            
\usepackage{booktabs}       
\usepackage{amsfonts}       
\usepackage{nicefrac}       
\usepackage{microtype}      
\usepackage{xcolor}         
\usepackage{wrapfig}  
\usepackage{enumitem}
\usepackage{fix-cm}

\usepackage{amsmath}
\usepackage{amssymb}
\usepackage{mathtools}
\usepackage{amsthm}
\usepackage{dsfont}

\usepackage{tcolorbox}
\tcbuselibrary{skins}

\allowdisplaybreaks[4]
\numberwithin{equation}{section} 

\newtheorem{remark}{Remark}
\newtheorem{theorem}{Theorem}
\newtheorem{assumption}{Assumption}
\newtheorem{definition}{Definition}
\newtheorem{proposition}{Proposition}

\newtheorem{lemma}{Lemma}
 
\title{Unveiling the Power of Multiple Gossip Steps: \\ A Stability-Based Generalization Analysis in Decentralized Training}

\author{%
  Qinglun Li$^{1}$, 
  Yingqi Liu$^{2}$, 
  Miao Zhang$^{1}$, 
  Xiaochun Cao$^{2}$, 
  Quanjun Yin$^{1,}$\thanks{Corresponding Authors}$\:\:$, 
  Li Shen$^{2,*}$\\
  $^1$State Key Laboratory of Digital Intelligent Modeling and Simulation,\\
  National University of Defense Technology,
  Changsha, 410073 \\
  $^2$School of Cyber Science and Technology, 
      Shenzhen Campus of Sun Yat-sen University, China\\  \texttt{liqinglun@nudt.edu.cn,yin\_quanjun@163.com,mathshenli@gmail.com} \\
}

\begin{document}

\maketitle

\begin{abstract}
\looseness = -1
Decentralized training removes the centralized server, making it a communication-efficient approach that can significantly improve training efficiency, but it often suffers from degraded performance compared to centralized training.
Multi-Gossip Steps (MGS) serve as a simple yet effective bridge between decentralized and centralized training, significantly reducing experiment performance gaps. 
However, the theoretical reasons for its effectiveness and whether this gap can be fully eliminated by MGS remain open questions.
In this paper, we derive upper bounds on the generalization error and excess error of MGS using stability analysis, systematically answering these two key questions.
1). \textit{Optimization Error Reduction}: MGS reduces the optimization error bound at an exponential rate, thereby exponentially tightening the generalization error bound and enabling convergence to better solutions.
2). \textit{Gap to Centralization}: Even as MGS approaches infinity, a non-negligible gap in generalization error remains compared to centralized mini-batch SGD ($\mathcal{O}(T^{\frac{c\beta}{c\beta +1}}/{n m})$ in centralized and  $\mathcal{O}(T^{\frac{2c\beta}{2c\beta +2}}/{n m^{\frac{1}{2c\beta +2}}})$ in decentralized).
Furthermore, we provide the first unified analysis of how factors like learning rate, data heterogeneity, node count, per-node sample size, and communication topology impact the generalization of MGS under non-convex settings without the bounded gradients assumption, filling a critical theoretical gap in decentralized training. Finally, promising experiments on CIFAR datasets support our theoretical findings.
\end{abstract}

\input{tex/1.introduction}
\input{tex/2.related}

\input{tex/3.method}

\input{tex/4.convergence}

\input{tex/5.experiment}

\input{tex/6.conclusion}

\bibliography{main.bib}
\bibliographystyle{unsrt}

\newpage

\clearpage

\input{tex/7.proof_mgs}

\end{document}

%% file: tex/1.introduction.tex
\section{Introduction}

\begin{wrapfigure}{r}{0.5\textwidth}   
\vspace{-2em}  
    \begin{center} 
    \includegraphics[width=0.5\textwidth]{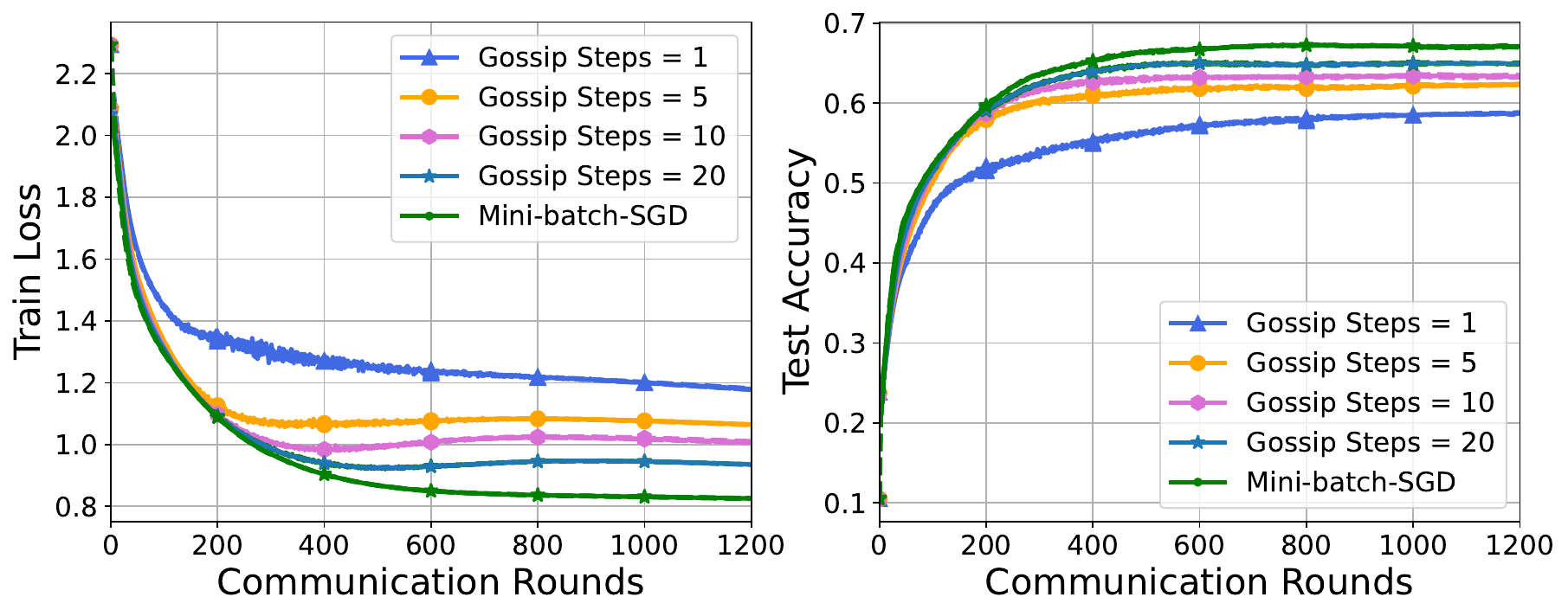}  
    \caption{ \small Under ring topology, DSGD-MGS with 20 gossip steps still shows significant performance gaps versus Mini-batch SGD in both training loss and test accuracy (LeNet on CIFAR-10, Dir 0.3, 50 nodes).}
    \label{fig:CSGD V.S. DSGD}
    \end{center}
    \vspace{-3em}  
\end{wrapfigure}

Recently, decentralized training \cite{chen2012diffusion,koloskova2020unified} has emerged as a promising alternative to centralized training, which suffers from challenges like high communication overhead \cite{li2025dfedadmm}, single point of failure \citep{chen2023enhancing}, and privacy risks \citep{gabrielli2023survey}. In contrast, decentralized training eliminates the central server, offering stronger privacy protection \citep{cyffers2022privacy}, faster model training \citep{lian2017can,koloskova2020unified}, and robustness to slow client devices \citep{neglia2019role}, making it an increasingly popular method \citep{chen2023enhancing, lian2017can}.

However, despite the aforementioned advantages of decentralized training, some works \cite{sun2023mode,liuunderstanding,bars2023improved} have pointed out that decentralized training methods underperform compared to centralized training methods in terms of model performance. Therefore, improving the performance of decentralized training models remains an important research question. Multiple Gossip Steps (MGS) \cite{sanghavi2007gossiping, deterministic2011gossip}, as a simple yet effective method to enhance the performance of decentralized training models, has been experimentally proven to significantly improve the efficiency and performance of decentralized training \cite{li2020communication, ye2021deepca, shi2023improving, ye2020decentralized}. Even under communication compression, MGS continues to demonstrate its advantages in performance improvement \cite{hashemi2021benefits}.

Despite the substantial empirical benefits of MGS, the underlying theoretical understanding of its efficacy and its potential to eliminate the performance gap with centralized training remain critical open questions. Specifically, two key issues need to be addressed:

\begin{tcolorbox}[
    enhanced,
    colback=gray!10!white, 
    colframe=gray!70!black, 
    boxrule=1pt, 
    arc=2pt 
]
(1) Why is MGS effective in improving model performance? \\
(2) Can decentralized training ultimately match or even surpass the performance of centralized training by increasing the number of gossip steps?
\end{tcolorbox}

To answer these open questions, we aim to theoretically explain how MGS works and how it affects model generalization. Using stability analysis, we find upper limits for the generalization error and excess error of MGS, giving systematic theoretical answers to these two main questions.

For Question 1, our theoretical analysis shows that MGS can reduce the optimization error bound at an exponential rate. This reduction in optimization error directly leads to an exponential reduction in the generalization error (as shown in Theorem \ref{the:optimization error of dsgd_mgs paper}, \ref{the:generalization error of dsgd_mgs}, and Remark \ref{remark:question1}), enabling the model to find better solutions. This relationship clearly explains why MGS effectively improves model performance. As illustrated in Figure \ref{fig:CSGD V.S. DSGD}, when the number of gossip steps is increased from 1 to 5, there is a significant reduction in the training loss (indicating reduced optimization error), and the test accuracy (measuring generalization) also shows a noticeable improvement. Furthermore, this improvement tends to diminish almost linearly as the number of gossip steps increases exponentially, consistent with the exponential decay in our theory findings.

For Question 2, our further analysis shows that even with a very large number of gossip steps, a basic difference in generalization error remains between decentralized DSGD-MGS and centralized mini-batch SGD.

Specifically, when the number of gossip steps becomes extremely large, the generalization error bound for DSGD-MGS becomes at most $\mathcal{O}(T^{\frac{2c\beta}{2c\beta +2}}/{n m^{\frac{1}{2c\beta +2}}})$. However, this is still noticeably larger than the centralized mini-batch SGD bound of $\mathcal{O}(T^{\frac{c\beta}{c\beta +1}}/{nm})$, highlighting a lasting difference in how it scales with the number of clients $m$ (because $1/m < 1/m^{\frac{1}{2c\beta +2}}$ when $m>1$).
This theoretical observation reveals a basic constraint: decentralized training cannot fully achieve the generalization performance of centralized training solely by increasing the number of MGS steps. Experiments shown in Figure \ref{fig:CSGD V.S. DSGD} support this conclusion, indicating that even with 20 gossip steps, DSGD-MGS still performs worse than centralized mini-batch SGD in the same settings.

\looseness=-1
Moreover, we are the first to provide a theoretical framework to understand how critical factors, including learning rate, data heterogeneity, number of nodes, sample size per node, and communication topology, jointly influence the generalization performance of MGS (see Reamrk \ref{remark:question1}-\ref{remark8}). Remarkably, we also eliminate the bounded gradient assumption in the non-convex condition. This work enhances our understanding of the challenges in decentralized learning and provides theoretical insights for hyperparameters to better model generalization. Finally, extensive experiments on CIFAR datasets further validate our theoretical results. The main contributions of this paper can be summarized as follows:

\begin{itemize}[leftmargin = 10pt]
\item Theoretically elucidating the mechanism by which MGS enhances the generalization performance of decentralized training models through an exponential reduction in optimization error.
\item Revealing that even with sufficient gossip communication, a theoretical gap in generalization error remains between MGS and centralized training, and this gap cannot be eliminated by MGS alone.

\item Establishing, for the first time under non-convex and without the bound gradient assumption, a unified framework analyzing factors impacting the MGS generalization performance (i.e., learning rate, data heterogeneity, number of nodes, sample size, and topology), thereby addressing a significant gap in existing theoretical frameworks.
\item Validating our theoretical findings through empirical experiments on the CIFAR datasets.
\end{itemize}
These findings provide new theoretical insights and practical implications for understanding and improving decentralized learning algorithms.

%% file: tex/2.related.tex
\section{Related Works}

This section reviews the current theoretical understanding and challenges in decentralized training, along with the evolution and impact of MGS. Moreover, at the end of each subsection, we highlight the existing gaps and open questions within these areas to position the contributions of this paper.

\textbf{Theoretical Analysis of D-SGD.}
\looseness = -1
Decentralized learning has attracted significant research interest due to its potential for enhanced privacy, communication efficiency, and scalability \citep{lian2017can, gabrielli2023survey,cyffers2022privacy,neglia2019role}. Early theoretical studies primarily focused on the convergence analysis of D-SGD, examining the number of iterations or communication rounds needed to reach an $\epsilon$-accurate solution \citep{lian2017can,hashemi2021benefits,kovalev2020optimal}. More recently, attention has shifted towards understanding the generalization performance of these algorithms. Sun et al. \cite{sun2021stability} were the first to analyze the generalization performance of D-SGD using uniform stability, later extending their results to asynchronous D-SGD \cite{deng2023stability}. However, these analyses assumed \textit{homogeneous data and bounded gradients}. Zhu et al. \cite{zhu2022topology} further studied the impact of communication topology on the generalization error of D-SGD, with their generalization bounds later improved by \cite{bars2023improved}, but they also relied on the same assumptions. More recently, Ye et al. \cite{ye2025icassp} analyzed the generalization behavior of D-SGD under heterogeneous data, but their analysis was limited to \textit{strongly convex} loss functions. Overall, current D-SGD theories still lack a unified framework that comprehensively accounts for all key algorithm parameters (e.g., data heterogeneity, non-convex loss function, topology, etc.).

\textbf{MGS in Decentralized Training.} Multiple Gossip Steps (MGS) \cite{catoni2007pac, sanghavi2007gossiping} is a technique that improves consensus by allowing multiple rounds of local communication. When integrated into decentralized algorithms, MGS not only enhances generalization performance but also accelerates convergence \cite{liu2011accelerated}. Additionally, Yuan et al. \cite{kovalev2020optimal} showed that MGS can reduce the adverse effects of data heterogeneity, a finding supported by other studies \cite{rogozin2021accelerated, shi2023improving}. Li et al. \cite{FedGosp} found that MGS can significantly improve algorithm accuracy.
In the field of decentralized federated learning, Shi et al. \cite{shi2023improving} incorporated MGS into their DFedSAM algorithm, significantly improving its generalization performance experimentally. Notably, MGS alone can achieve optimal convergence rates in non-convex settings \cite{kovalev2020optimal} without relying on more complex techniques like gradient tracking \cite{lu2021optimal}, quasi-global momentum \cite{lin2021quasi}, or adaptive momentum \cite{nazari2022dadam}. However, these studies have largely overlooked the question of why MGS is effective from a generalization perspective, with these advantages demonstrated mainly through empirical results, leaving a significant gap in the theoretical understanding of MGS.

%% file: tex/3.method.tex
\section{Background}

In this section, we first present some fundamental definitions required for stability analysis, including population risk, empirical risk, generalization error, excess error, and $l_2$ on-average model stability. Subsequently, we introduce a key lemma that establishes the relationship between the generalization error bound and $l_2$ on-average model stability.

\subsection{Stability and Generalization in Decentralized Learning}

\label{sec:background}

We consider the general statistical learning setting, adapted to a decentralized framework with $m$ agents\footnote{In this paper, the terms node, agent, and client are used interchangeably.}. Each agent $k$ observes data points drawn from a local distribution $\mathcal{D}_k$ with support $\mathcal{Z}$. The goal is to find a global model $\theta \in \mathbb{R}^d$ that minimizes the \emph{population risk}, defined as:

$$R(\theta)\triangleq \frac{1}{m}\sum_{k=1}^ml_{k}(\theta)  \triangleq \frac{1}{m}\sum_{k=1}^m\mathbb{E}_{Z \sim \mathcal D_k} [\ell(\theta; Z)]\;,$$
where $\ell$ is some loss function. We denote by $\theta^\star$ a global minimizer of the population risk, i.e., $\theta^\star\in \arg \min_ \theta R(\theta)$.

Although the population risk $R(\theta)$ is not directly computable, we can instead evaluate an empirical counterpart using $m$ local datasets $S \triangleq (S_1, \ldots, S_m)$, where $S_k = \{Z_{1k}, \ldots, Z_{nk}\}$ represents the dataset of agent $k$, with each sample $Z_{ik}$ drawn from the local distribution $\mathcal{D}_k$. For simplicity, we assume that each local dataset has the same size $n$, though our analysis can be extended to the heterogeneous case. The resulting \emph{empirical risk} is given by:
$$R_S(\theta)\triangleq \frac{1}{m}\sum_{k=1}^m R_{S_k}(\theta)\triangleq \frac{1}{mn}\sum_{k=1}^m\sum_{i=1}^n \ell(\theta; Z_{ik})\;.$$

\looseness=-1
One of the most well-known and extensively studied estimators is the empirical risk minimizer, defined as $\widehat{\theta}_{\text{ERM}} \triangleq \arg\min_\theta R_S(\theta)$. However, in most practical scenarios, directly computing this estimator is infeasible. Instead, one typically employs a potentially random \emph{decentralized optimization} algorithm $A$, which takes the full dataset $S$ as input and returns an approximate minimizer $A(S) \in \mathbb{R}^d$ for the empirical risk $R_S(\theta)$.

In this setting, the expected \emph{excess risk} $R(A(S)) - R(\theta^\star)$ can be upper-bounded by the sum of the (expected) \emph{generalization error} ($\epsilon_{\textrm{gen}}$) and the (expected) \emph{optimization error} ($\epsilon_{\textrm{opt}}$) \cite{ye2025icassp,bars2023improved}:
\begin{equation}\label{eq:connection gen and opt and excess error}
    \mathbb{E}_{A,S}  [R(A(S)) - R(\theta^\star)] \leq \epsilon_{\textrm{gen}} + \epsilon_{\textrm{opt}} 
\end{equation}
where $\epsilon_{\textrm{gen}} \triangleq \mathbb{E}_{A,S}[R(A(S)) - R_S(A(S))]$ and $\epsilon_{\textrm{opt}} \triangleq \mathbb{E}_{A,S}[R_S(A(S)) - R_S(\widehat{\theta}_{\text{ERM}})]$.
This work focuses on controlling the expected generalization error $\epsilon_{\textrm{gen}}$, for which a common approach is to use the stability analysis of the algorithm $A$.

Contrary to a large body of works using the well-known \emph{uniform stability} \citep{bousquet2002stability,shalev2010learnability}, our analysis relies on the notion of \emph{on-average model stability} \citep{lei2020fine}, which has the advantage of removing the bounded gradient assumption \cite{li2025dfedadmm,Sun2022Decentralized,liuunderstanding} in our analysis, making the theoretical results more general. Below, we recall this notion, with a slight adaptation
to the decentralized setting.

\begin{definition}[\textbf{$l_2$ on-average model stability}] \label{def:on-average} Let $S = (S_1,\ldots,S_m)$ with $S_k = \{Z_{1k},\ldots,Z_{nk}\}$ and $\Tilde{S} = (\Tilde{S}_1,\ldots,\Tilde{S}_m)$ with $\Tilde{S}_k = \{\Tilde{Z}_{1k},\ldots,\Tilde{Z}_{nk}\}$ be two independent copies such that $Z_{ik}\sim \mathcal{D}_k$ and $ \Tilde{Z}_{ik}\sim \mathcal{D}_k$. For any $i\in\{1,\ldots, n\}$ and $j\in\{1,\ldots,m\}$, let us denote by $S^{(ij)}=(S_1,\ldots,S_{j-1},S_j^{(i)},S_{j+1},\ldots,S_m)$, with $S_j^{(i)}=\{Z_{1j},\ldots,Z_{i-1j},\Tilde{Z}_{ij},Z_{i+1j},\ldots,Z_{nj}\}$, the dataset formed from $S$ by replacing the $i$-th element of the $j$-th agent's dataset by $\Tilde{Z}_{ij}$. A randomized algorithm $A$ is said to be \emph{$l_2$ on-average model $\varepsilon$-stable} if
\begin{equation}
    \mathbb{E}_{S,\Tilde{S},A}\Big[\frac{1}{mn}\sum_{i=1}^n\sum_{j=1}^m ||A(S) - A(S^{(ij)})||_2^2\Big] \leq \varepsilon^2 \;.
\end{equation}

\end{definition}

A key aspect of on-average model stability is that it can directly be linked to the generalization error, as shown in the following lemma.

\begin{lemma}[\textbf{Generalization via on-average model stability \citep{lei2020fine}}] 
\label{lemma:ob-avg-gen} Let $A$ be $l_2$ on-average model $\varepsilon$-stable. Let $\gamma > 0$.
Then, if $\ell(\cdot;z)$ is nonnegative and is $\beta$-smoothness for all $z\in\mathcal{Z}$, we have 
\begin{equation*}
    \epsilon_{\textrm{gen}}\leq \frac{1}{2mn\gamma}\sum_{i=1}^n\sum_{j=1}^m\mathbb{E}_{A,S}[\|\nabla\ell(A(S);Z_{ij})\|^2] + \frac{\beta+\gamma}{2mn}\sum_{i=1}^n\sum_{j=1}^m\mathbb{E}_{A,\Tilde{A},S}[\|A(S) - A(S^{(ij)})\|^2]
\end{equation*}

\end{lemma}
In fact, we modified the proof of the lemma from Lei et al.\cite{lei2020fine}, replacing the $R_S(A(S))$ on the right-hand side with a gradient $\mathbb{E}_{A,S}[\|\nabla\ell(A(S);Z_{ij})\|^2]$. This adjustment better captures the impact of data heterogeneity on the generalization error.
With this lemma, obtaining the desired generalization bound reduces to controlling the $l_2$ on-average model stability of the decentralized algorithm $A$.

\subsection{Decentralized SGD with Multiple Gossip Steps}
\label{sec:dsgd}

In this paper, we focus on the widely-used Decentralized Stochastic Gradient Descent (D-SGD) algorithm \citep{nedic2009distributed,lian2017can}, which aims to find minimizers (or saddle points) of the empirical risk $R_S(\theta)$ in a fully decentralized manner. This algorithm relies on peer-to-peer communication between agents, with a graph representing which pairs of agents (or nodes) are able to interact. Specifically, the \emph{communication topology} is captured by a gossip matrix $W \in [0,1]^{m \times m}$ (see Definition \ref{def:gossip_matrix}), where $W_{jk} > 0$ indicates the weight that agent $j$ assigns to messages from agent $k$, and $W_{jk} = 0$ (no edge) implies that agent $j$ does not receive messages from agent $k$.

\begin{wrapfigure}{r}{0.5\textwidth}
\vspace{-1em}
\begin{minipage}{0.5\textwidth}
\hfill
\begin{algorithm}[H]
    \caption{Decentralized SGD with MGS}\label{alg:dsgd-mgs} 
    \begin{algorithmic}[1]
        \STATE {\bfseries Input:} Initialize $\forall k$, $\theta_k^{(0)} = \theta^{(0)} \in \mathbb{R}^d$, iterations $T$, stepsizes $\{\eta_t\}_{t=0}^{T-1}$, weight matrix $W$, Multiple Gossip Steps $Q$.
        \FOR{$t=0,\ldots,T-1$}
        \FOR{each node $k=1,\ldots,m$ in parallel }
            \STATE {\texttt{Local Update Steps:}}
            \STATE Sample $I^t_k\sim \mathcal{U}\{1,\ldots, n\}$ 
            \STATE $\theta_k^{(t,0)} = \theta_k^{(t)} - \eta_t\nabla \ell(\theta_k^{(t)};Z_{I^t_kk})$ 
            \STATE {\texttt{Multiple Gossip Steps:}}
            \FOR{$q = 0$ to $Q-1$}
            \STATE $\theta_k^{(t,q+1)} = \sum^m_{l=1}W_{kl}\theta_l^{(t,q)}$
            \ENDFOR
            \STATE $\theta_k^{(t+1)} = \theta_k^{(t,Q)}$
        \ENDFOR
        \ENDFOR
    \end{algorithmic}
\end{algorithm}
\end{minipage}
\vspace{-3em}
\end{wrapfigure}

The D-SGD with Multiple Gossip Steps (DSGD-MGS) algorithm performs multiple gossip updates during the communication phase of the D-SGD algorithm, while all other computational components remain identical to D-SGD, as detailed in Algorithm \ref{alg:dsgd-mgs}.
Specifically, the main procedure at time $t$ is divided into two steps:

\vspace{-0.5em}
\begin{itemize}[left=0pt, itemsep=5pt, parsep=0pt,after=\vspace{-10pt}]
    \item \textbf{Local Update Steps:} Each node independently and uniformly draws a training sample $Z_{I^t_kk}$ from its local dataset $S_k$. Based on the current model parameter $\theta_k^{(t)}$, it computes the gradient $\nabla \ell(\theta_k^{(t)}; Z_{I^t_kk})$ and performs gradient descent to obtain the initial point for Multiple Gossip Steps: $\theta_k^{(t,0)} = \theta_k^{(t)} - \eta_t \nabla \ell(\theta_k^{(t)}; Z_{I^t_kk})$, where $\eta_t$ denotes the step size.
    \looseness=-1
    \item \textbf{Multiple Gossip Steps:} Each node exchanges information with its neighbors through $Q$ gossip averaging steps: $\theta_k^{(t,q+1)} = \sum_{l=1}^m W_{kl} \theta_l^{(t,q)}$. The resulting model parameter $\theta_k^{(t,q+1)}$ is then used as the initial point $\theta_k^{(t+1)}$ for the next Local Update Steps.
\end{itemize}

%% file: tex/4.convergence.tex
\section{Generalization Analysis}

In this section, we first introduce the Definition and Assumptions required for analyzing the generalization of the DSGD-MGS algorithm. We then present the upper bounds for the generalization error and excess error, followed by a detailed analysis of these bounds. Proofs for all Lemmas and Theorems can be found in the \textbf{Appendix} \ref{sec_appendix:proof of all}.

\subsection{Definition and Assumption}

\begin{definition}[Gossip Matrix]\label{def:gossip_matrix}
Let $ W \in [0, 1]^{n \times n} $ be a symmetric doubly stochastic matrix. This means that $ W = W^\top $, and both the row sums and column sums of $ W $ equal one, i.e., $ W \mathbf{1} = \mathbf{1} $ and $ \mathbf{1}^\top W = \mathbf{1}^\top $, where $ \mathbf{1} $ is the vector of all ones. The eigenvalues of $ W $ are ordered as $ 1 = |\lambda_1(W)| > |\lambda_2(W)| \geq \cdots \geq |\lambda_n(W)| $. The spectral gap of $ W $, denoted by $ \delta $, is defined as $\delta := 1 - |\lambda_2(W)|\in (0, 1) $.
\end{definition}

\begin{assumption}\emph{($\beta$-smoothness).} \label{ass:smooth}The loss function $\ell$ is $\beta$-smooth i.e. $\exists \beta>0$ such that $\forall \theta, \theta' \in \mathbb{R}^d, z \in \mathcal{Z}$, $\|\nabla\ell(\theta;z)-\nabla\ell(\theta';z)\|_2\leq \beta \|\theta - \theta'\|_2$.    
\end{assumption}

\begin{assumption}\emph{(Bounded Stochastic Gradient Noise).} \label{ass:bound SG} There exists $\sigma^2>0$ such that $\mathbb{E}_{Z_{i,j}}\|\nabla \ell(\theta;Z_{i,j})-\nabla R_{\mathcal{S}_j}(\theta)\|^2\leq\sigma^2$, for any agent $j \in [m]$ and $\theta \in \mathbb{R}^d$.
\end{assumption}

\begin{assumption}\emph{(Bounded Heterogeneity).}\label{ass:bound_hetero} There exists $\xi^2 > 0$ such that $\frac{1}{m}\sum_{k=1}^m\|\nabla R_{S_k}(\theta) - \nabla R_S(\theta)\|^2 \le \xi^2$, for any $\theta \in \mathbb{R}^d$. 
\end{assumption}

Using the property $\beta$-smoothness of $\ell(\theta; z)$, it is straightforward to show that $\ell_{k}(\theta) = \mathbb{E}_{Z \sim \mathcal{D}_k} [\ell(\theta; Z)]$ and $R_{S_k}(\theta) = \frac{1}{n}\sum_{i=1}^n\ell(\theta;Z_{ik})$ also satisfy the property $\beta$-smoothness. 

\begin{remark}
    Definition \ref{def:gossip_matrix} stipulates that the communication topology must be a doubly stochastic matrix, which appears in many decentralized optimization works~\cite{lian2017can,Sun2022Decentralized,bars2023improved,hashemi2021benefits,li2025dfedadmm}. Assumption \ref{ass:smooth} specifies that the loss function is smooth, which is often used in optimization and generalization studies under non-convex settings \cite{li2025dfedgfm,liuunderstanding,liu2024decentralized,li2024boosting,li2024oledfl,li2025asymmetrically}. Assumption \ref{ass:bound SG} states that the stochastic gradients of the samples are bounded, and Assumption \ref{ass:bound_hetero} bounds the heterogeneity of the data. These assumptions are frequently used in the convergence analysis of many works~\cite{li2025dfedgfm,li2025dfedadmm,shi2023improving,liu2024decentralized}, and we will employ them in this paper to analyze the stability and generalization of DSGD-MGS.
\end{remark}

\subsection{Generalization Error and Excess Error of DSGD-MGS}

Due to its fully decentralized structure, DSGD-MGS produces $m$ distinct outputs, $A_1(S) \triangleq \theta_1^{(T)}, \ldots, A_m(S) \triangleq \theta_m^{(T)}$, one for each agent. As a result, the stability and generalization analysis that follows will focus on these individual outputs, rather than a single global output $A(S)$ as described in Section \ref{sec:background}.
Denote by $A_k(S) = \theta_k^{(T)}$ and  $A_k(S^{(ij)}) = \theta_k^{(T)}(i,j)$, the final iterates of agent $k$ for DSGD-MGS run over two data sets $S$ and $S^{(ij)}$ that differ only in the $i$-th sample of agent $j$. To obtain a tighter upper bound for the non-convex case, we modify Lemma \ref{lemma:ob-avg-gen} by introducing a variable $t_0$, resulting in the following key lemma, which transforms the computation of the generalization error upper bound $\epsilon_{\textrm{gen}}$ into the computation of the stability upper bound.

\begin{lemma} \label{lemma:key-non-conv_paper} Assume the loss function $\ell(\cdot, z)$ is nonnegative and bounded in $[0,1]$, and that Assumptions \ref{ass:smooth} hold. For all $i=1,\ldots,n$ and $j=1,\ldots, m$, let $\{\theta_k^{(t)}\}_{t=0}^T$ and $ \{\tilde{\theta}_k^{(t)}(i,j)\}_{t=0}^T$, the iterates of agent $k = 1,\ldots, m$ for DSGD-MGS run on $S$ and $S^{(ij)}$ respectively. Then, for every $t_0 \in \{0, 1, \ldots, T\}$ we have:
\begin{align*}
    &|\mathbb{E}_{A,S}[R(A_k(S)) - R_S(A_k(S))]| \\
    \quad &\leq \frac{t_0}{n} +\underbrace{\frac{\gamma + \beta}{2mn}\sum_{i=1}^n\sum_{j=1}^m\mathbb{E}[\delta_k^{(T)}(i,j) \big| \delta^{(t_0)}(i,j) = \mathbf{0}]}_{I_1:\,l_2 \,\text{on-average model stability}} +
    \underbrace{\frac{1}{2mn\gamma}\sum_{i=1}^n\sum_{j=1}^m\mathbb{E}[\|\nabla\ell(A_k(S);Z_{ij})\|^2]}_{I_2:\,\text{Related to optimization error}}
\end{align*}
where $\delta^{(t)}(i,j)$ is the vector containing $\forall k=1,\ldots,m$, $\delta_k^{(t)}(i,j) = \|\theta_k^{(t)} - \tilde{\theta}_k^{(t)}(i,j)\|_2^2$.
\end{lemma}
According to Lemma \ref{lemma:key-non-conv_paper}, to compute the generalization error $\epsilon_{\textrm{gen}}$, We need to calculate the $l_2$ on-average model stability ($I_1$) and the gradient related to the optimization error ($I_2$). Below, we first provide the stability upper bound, followed by the optimization error upper bound.

\textbf{Upper bound of $I_1$:}
For a fixed couple $(i,j)$, we are first going to control the vector $\Delta^{(t)}=\frac{1}{mn}\sum_{i,j}\Delta^{(t)}(i,j)$, where $\Delta^{(t)}(i,j) \triangleq\mathbb{E}[\delta^{(t)}(i,j) | \delta^{(t_0)}(i,j) = \mathbf{0}]$. When it is clear from context, we simply write $\tilde{\theta}_k^{(t)}(i,j) = \tilde{\theta}_k^{(t)}$. Next, we provide the upper bound of the $l_2$ on-average model stability for the DSGD-MGS algorithm.
\begin{theorem}[\textbf{Stability for the DSGD-MGS}]\label{the:stability of dsgd_mgs}  
As in the conditions of Lemma \ref{lemma:key-non-conv_paper}, then the following holds:
\begin{align*}
     \frac{1}{mn}\sum_{i=1}^n\sum_{j=1}^m\mathbb{E}[\delta_k^{(T)}(i,j) \big| \delta^{(t_0)}(i,j) =\mathbf{0}] \leq \frac{8e\sqrt{2\beta}c^2}{(1+2c\beta)nmt_0} \left(\frac{T}{t_0}\right)^{2c\beta}
\end{align*}
\end{theorem}

\textbf{Upper bound of $I_2$:} 
Let $\bar{G} = \frac{1}{mn} \sum_{i,j}\mathbb{E}[\|\nabla\ell(\theta_k^{(T)};Z_{ij})\|^2]$. According to the Assumptions~\ref{ass:bound SG} and \ref{ass:bound_hetero}, the following inequality holds:
\begin{align*}
    \bar{G} &= \frac{1}{mn} \sum_{i,j}\mathbb{E}[\|\nabla\ell(\theta_k^{(T)};Z_{ij})\|^2] = \frac{1}{mn} \sum_{i,j}\mathbb{E}[\|\nabla\ell(\theta_k^{(T)};Z_{ij}) \pm \nabla R_{S_k}(\theta_k^{(T)})\pm \nabla R_S(\theta_k^{(T)})\|^2] \\
    &\leq 3\sigma^2 + 3\xi^2 + 3\mathbb{E}[\|\nabla R_S(\theta_k^{(T)})\|^2]
\end{align*}
Since $\ell$ satisfies the $\beta$-smoothness property, it is straightforward to show that $R_{S}(\theta_k^{(T)})$ also satisfies the $\beta$-smoothness property. Consequently, $R_{S}(\theta)$ also satisfies the self-bounding property in Lemma \ref{le:self-bound} (see the \textbf{Appendix} \ref{sec_appendix:proof of all}), i.e., $\|\nabla R_{S}(\theta)\|^2 \leq 2\beta R_{S}(\theta)$. Then, we have
\begin{align}\label{eq:gradient_bound_paper}
    \bar{G} \le 3\sigma^2 + 3\xi^2 + 6\beta\mathbb{E}_S[R_S(\theta_k^{(T)})]
\end{align}

Next, we will focus on bounding $\mathbb{E}_S[R_S(\theta_k^{(T)})]$.
According to the results from \cite[Theorem 1]{hashemi2021benefits} (see Lemma \ref{the:optimization error of dsgd_mgs paper} in the \textbf{Appendix} \ref{sec_appendix:proof of all}), we have the following theorem:

\begin{theorem}
[\textbf{Optimization error of DSGD-MGS}]\label{the:optimization error of dsgd_mgs paper}
\looseness=-1
Let  $\Delta^2 := \max_{\mathbf{\theta}^* \in \mathcal{X}^*} \sum_{k=1}^m \|\nabla R_{S_k}(\theta^*)\|^2$, $R_0 := R_{S}(\theta^{(0)}) - R_S^*$, where $\mathcal{X}^* = \arg\min_{\theta} R_S(\theta)$ and $R_S^* = R_{S}(\widehat{\theta}_{\text{ERM}})$. Suppose Assumptions \ref{ass:smooth} and Polyak-Łojasiewicz (PL) condition (see Assumption \ref{ass:PL-condition} in the \textbf{Appendix}) hold. Define
\begin{align*}
Q_{0}  :=\log{(\bar{\rho}/46)}/\log{\left(1-\frac{\delta\tilde{\gamma}}{2}\right)},\bar{\rho}:=1-\frac{\mu}{m\beta}, 
\tilde{\gamma}  =\frac{\delta}{\delta^2 + 8\delta+(4+2\delta)\lambda_{\max}^2(I-W)}.
\end{align*}
Then, if the nodes are initialized such that $\theta_k^Q=0$, for any $Q>Q_0$ after $T$ iterations the iterates of DSGD-MGS with $\eta_t = \frac{1}{\beta}$ satisfy
\begin{align}\label{eq:opt error}
\mathbb{E}_{S}[R_S(\theta_k^{(T)})]-R_S^*=\mathcal{O}\left(\frac{\Delta^2e^{-\frac{\delta \tilde{\gamma} Q}{4}}} {1-\bar{\rho}} + \left[1+\frac{\beta}{\mu\bar{\rho}}\left(1+e^{-\frac{\delta \tilde{\gamma} Q}{4}}\right)\right]R_0\rho^T\right).
\end{align}
Here, $\delta$ represents the spectral gap of $W$, and $\rho \triangleq 1 - \delta = |\lambda_2(W)|$ is defined in definition \ref{def:gossip_matrix}.
\end{theorem}

By combining Equation (\ref{eq:gradient_bound_paper}) with Theorem \ref{lemma:key-non-conv_paper}, we obtain the upper bound for $\bar{G}$.
\begin{align}\label{eq:bar_G}
    \bar{G} = \mathcal{O}(\sigma^2 + \delta^2 + R_S^*) + \mathcal{O}\left(\frac{\beta \Delta^2e^{-\frac{\delta \tilde{\gamma} Q}{4}}} {1-\bar{\rho}} + \left[1+\frac{\beta}{\mu\bar{\rho}}\left(1+e^{-\frac{\delta \tilde{\gamma} Q}{4}}\right)\right]R_0\beta\rho^T\right)
\end{align}

\textbf{Generalization Bound for DSGD-MGS:} With the above Theorem \ref{the:stability of dsgd_mgs} \& \ref{the:optimization error of dsgd_mgs paper}, we can derive the generalization error upper bound for DSGD-MGS.
\begin{theorem}[\textbf{Generalization error of DSGD-MGS}]\label{the:generalization error of dsgd_mgs} Based on Lemma~\ref{lemma:key-non-conv_paper}, Theorem~\ref{the:stability of dsgd_mgs} and Theorem~\ref{the:optimization error of dsgd_mgs paper}, and assuming that Assumptions~\ref{ass:smooth}-\ref{ass:bound_hetero} hold, let the learning rate satisfy $\eta_t \leq \frac{c}{t+1}$ for some constant $c > 0$. We derive the following result by appropriately selecting $t_0$ and $\gamma$:
\begin{align*}
    &|\mathbb{E}_{A,S}[R(A_k(S)) - R_S(A_k(S))]| \\
    &\hspace{3em} \le \frac{2c\beta+3}{\left(n(2c\beta+1)\right)^{\frac{2c\beta+2}{2c\beta+3}}} \left( \frac{2 \bar{G} e\sqrt{2\beta}c^2T^{2c\beta}}{m} \right)^{\frac{1}{2c\beta+3}}  + \frac{2c\beta+2}{n(2c\beta+1)} \left( \frac{4\beta e\sqrt{2\beta}c^2T^{2c\beta}}{m} \right)^{\frac{1}{2c\beta+2}} 
\end{align*}
where the expression for $\bar{G}$ is given in Equation (\ref{eq:bar_G}). 
\end{theorem}

\begin{remark}
    [\textbf{Optimization Error Reduction}]\label{remark:question1}
    As shown in Theorem \ref{the:generalization error of dsgd_mgs}, the generalization error bound obtained via $l_2$ on-average model stability is closely related to the optimization error $\bar{G}$. Analyzing the MGS-related terms reveals that increasing the number of MGS steps $Q$ reduces $\bar{G}$, thereby tightening the generalization error bound. Moreover, a more detailed analysis shows that the reduction in the generalization error bound is exponential, specifically on the order of $\mathcal{O}(e^{-\frac{\delta \gamma Q}{4}})$, indicating that even a small increase in $Q$ can lead to significant gains. This observation will also be validated in the experimental section \ref{sec:experiment of generalization error}.
\end{remark}

\begin{remark}
    [\textbf{Gap to Centralization}]
    \looseness = -1
    As indicated by Theorem \ref{the:generalization error of dsgd_mgs}, by letting $Q$ approach infinity, we can derive the limiting generalization error bound, which helps address whether DSGD-MGS with sufficiently many steps can effectively approximate centralized mini-batch SGD. The answer is no, because the resulting bound is at most $\mathcal{O}\left(T^{\frac{2c\beta}{2c\beta +2}} / n m^{\frac{1}{2c\beta +2}}\right)$, which still differs in terms of node count $m$ and per-node data size $n$ from the bound $\mathcal{O}\left(T^{\frac{c\beta}{c\beta + 1}} / mn\right)$ established for centralized mini-batch SGD based on uniform stability in \cite{bars2023improved,hardt2016train}. Therefore, this gap persists unless the number of nodes or the data size per node is significantly increased. As illustrated in Figure \ref{fig:CSGD V.S. DSGD}.
\end{remark}

\begin{remark}
    [\textbf{Related to the Optimization Error}]
    \looseness = -1
    Compared to prior works on the generalization error of D-SGD \cite{Sun2022Stability,hardt2016train,bars2023improved,zhu2022topology}, which rely on Lipschitz assumptions for the loss function, our approach removes this assumption, allowing for a more explicit connection between optimization error and generalization error. In those works, the Lipschitz assumption effectively absorbs optimization-related quantities (e.g., gradients) into a Lipschtiz constant, obscuring this relationship. In contrast, our work removes the Lipschitz assumption, making the relationship between generalization and optimization errors more explicit. Our results show that reducing optimization error can also decrease generalization error to some extent, which explains the common observation that as training progresses, both the training error decreases and the model's performance on the validation set improves.
\end{remark}

\begin{remark}[\textbf{Influential Factors of the Generalization Error for DSGD-MGS}]\label{remark5:factor of gen error}
\looseness = -1
    When the model, loss function, and dataset are fixed, parameters like the smoothness $\beta$, gradient noise $\sigma$, and data heterogeneity $\delta$ are also fixed. In this case, to reduce the generalization error bound according to the upper bound in Theorem \ref{the:generalization error of dsgd_mgs}, the following strategies are effective:
    \: 1) Increase the data size per node $n$;
    \: 2) Increase the number of nodes $m$;
    \: 3) Increase the MGS step count $Q$;
    \: 4) Reduce the distance between the optimal point and the initial point $R_0$;
    \: 5) Use a communication topology with a larger spectral gap $\delta$ (which implies a smaller $\rho$);
    \: 6) Decrease the learning rate $c$.
    The first five are straightforward, while the sixth is recommended because the number of iterations $T$ is usually large, making $T^{\frac{2c\beta}{2c\beta + 2}}$ the dominant term in the bound. Reducing $c$ can significantly reduce this term. Additionally, if the choice of dataset is flexible, selecting one that is as close to i.i.d. as possible is beneficial, as a larger data heterogeneity parameter $\xi$ will generally increase the generalization error bound.
\end{remark}

\begin{remark}
    [\textbf{Innovation in Generalization Error Bounds}]
    \looseness = -1
    Our work introduces $l_2$ on-average model stability to deriving generalization error bounds for decentralized algorithms, characterized by the following key innovations: 
    \: 1) Removal of Lipschitz Assumption: 
    Unlike previous proofs based on uniform stability \cite{hardt2016train,sun2021stability,bars2023improved,liuunderstanding,sun2023mode,sun2023understanding}, our approach removes the Lipschitz assumption on the loss function (which implicitly bounds the gradient), allowing the relationship between optimization error and generalization error to become more explicit.
    \: 2) Explicit Role of Optimization Error: 
    We establish, for the first time, a direct connection between the optimization error and generalization error of the D-SGD algorithm, revealing that reducing the optimization error also decreases the generalization error, which aligns better with observed training dynamics.
    \: 3) Exponential MGS Benefit: Our bounds demonstrate that the impact of MGS on reducing generalization error is exponential, highlighting the significant gains achievable with a moderate number of MGS steps.
    \: 4) Quantification of Heterogeneity Impact: 
    Ye et al.\cite{ye2025icassp} were the first to theoretically reveal that data heterogeneity can degrade the generalization bound of the D-SGD algorithm under the strongly convex setting. Building on this, we take a further step by providing a precise characterization of how data heterogeneity affects generalization in the non-convex setting, filling a critical gap in existing theoretical analyses.
\end{remark}

\begin{theorem}[\textbf{Excess Error of DSGD-MGS}]\label{the:excess error}
    Under the same conditions and notation as Theorems \ref{the:generalization error of dsgd_mgs} and \ref{the:optimization error of dsgd_mgs paper}, and based on the decomposition of excess error in Equation (\ref{eq:connection gen and opt and excess error}), the optimization error bound (Equation \ref{eq:opt error}), and the generalization error bound (Theorem \ref{the:generalization error of dsgd_mgs}), we obtain the following upper bound for the excess error.
    \begin{align}
        \mathbb{E}_{A,S}  [R(A(S)) - R(\theta^\star)] & = \mathcal{O}\Bigg(\frac{\Delta^2e^{-\frac{\delta \gamma Q}{4}}} {1-\bar{\rho}} + \left[1+\frac{\beta}{\mu\bar{\rho}}\left(1+e^{-\frac{\delta \gamma Q}{4}}\right)\right]R_0\rho^T\\ \nonumber
        & + \frac{1}{n^{\frac{2c\beta+2}{2c\beta+3}}} \left( \frac{\bar{G} \beta^{\frac{3}{2}}c^3T^{2c\beta}}{m} \right)^{\frac{1}{2c\beta+3}}  + \frac{1}{n} \left( \frac{\beta^{\frac32}c^2T^{2c\beta}}{m} \right)^{\frac{1}{2c\beta+2}} \Bigg)
    \end{align}
\end{theorem}

\begin{remark}
    [\textbf{The difference of conclusions obtained from excess error and generalization error}]
    \label{remark:lr_difference}
    \looseness = -1
    Since the excess error can be decomposed as $\mathbb{E}_{A,S} [R(A(S)) - R(\theta^\star)] \leq \epsilon_{\textrm{gen}} + \epsilon_{\textrm{opt}}$, most conclusions about the generalization error also apply to the excess error (see Remark \ref{remark5:factor of gen error}). The only key difference lies in the choice of learning rate.
    For $\epsilon_{\textrm{gen}}$, a smaller learning rate (i.e., smaller $c$) is preferred, as $\epsilon_{\textrm{gen}}$ is dominated by the term $\mathcal{O}(T^{\frac{2c\beta}{2c\beta + 2}})$, meaning that reducing $c$ significantly reduces this term and hence the generalization error.
    However, this is not the case for $\epsilon_{\textrm{opt}}$. Prior work on the convergence of D-SGD \cite{lian2017can} shows that $\epsilon_{\textrm{opt}} = \mathcal{O}\left(\frac{R_0}{T\eta}\right)$, indicating that an excessively large learning rate increases $\epsilon_{\textrm{opt}}$, thereby undermining convergence.
    Thus, the choice of learning rate involves a trade-off between minimizing generalization error and maintaining convergence, a conclusion that will be confirmed in the Experimental Section \ref{sec:sec:experiment of excess error}.
\end{remark}

\begin{remark}
\textbf{(On the Technical Role of the PL Condition).}
    Our analysis of the generalization error requires bounding the expected squared gradient norm at the final iterate, denoted as $\bar{G}$. However, establishing a tight upper bound for the final iterate's gradient in non-convex decentralized optimization remains a challenging frontier problem. While recent advances have been made in last-iterate convergence analysis (e.g., \citep{yuan2022revisiting}), existing results either do not incorporate the MGS mechanism or provide bounds only on the function value gap, which are insufficient for directly bounding $\bar{G}$. To bridge this gap, we adopt the Polyak-Łojasiewicz (PL) condition. This is a standard approach in the literature (e.g., \citep{Sun2022Decentralized}) used to connect the squared gradient norm with the function value gap. This technical choice is deliberate and crucial, as a tight upper bound on the function value gap under the MGS setting is available \citep{hashemi2021benefits}. Consequently, the PL condition enables us to derive some of the first fine-grained, MGS-aware generalization bounds that explicitly link the generalization error to key algorithmic hyperparameters, including the number of MGS steps ($Q$), communication topology, and learning rate. This provides concrete, quantitative insights that significantly advance beyond high-level bounds, such as the classic $\mathcal{O}(1/T)$ analysis provided by L2-stability \citep{lei2020fine}. Therefore, the reliance on the PL condition reflects the current theoretical limits in non-convex last-iterate analysis rather than a fundamental limitation of our stability framework. Our framework is modular: should future research provide a direct, assumption-free upper bound for $\bar{G}$ in the MGS setting, our generalization bounds can be immediately strengthened by replacing this component. A more detailed discussion is provided in the \textbf{Appendix} \ref{appendix:PL-condition}.
\end{remark}

\begin{remark}\label{remark8}
    All the above discussions are also solid to \( \bar{\theta}^{(T)} \triangleq \frac{1}{m}\sum_{k=1}^m\theta_{k}^{(T)}\). In addition, our theoretical results apply to decentralized topologies other than the fully connected case. When the topology becomes fully connected, the iterative update reduces to the centralized setting. For detailed analysis, please refer to the \textbf{Appendix}. For detailed proof, please refer to \textbf{Appendix} \ref{appendix_sec:bar{theta}}. Additionally, we provide a \textbf{consensus error analysis} to further illustrate the behavior of MGS in both finite and infinite regimes (detailed discussion provided in \textbf{Appendix} \ref{sec:appendix_error_consensus}). Furthermore, we extend our theoretical analysis to the case involving batch size $b$. The detailed proofs and analyses are provided in the \textbf{Appendix} \ref{appendix:proof of minibatch}.
\end{remark}

%% file: tex/5.experiment.tex
\section{Experiment}

In this section, we present extensive experiments to validate our theoretical findings. We first describe the experimental setup, followed by the empirical results and corresponding analysis. Due to space constraints, the experimental validation of excess error is presented in \textbf{Appendix} \ref{sec:sec:experiment of excess error}. Furthermore, we conduct an in-depth exploration of the subtle relationship between mini-batch size and (Q) on the CIFAR-100 dataset, providing practitioners with insights for achieving higher performance. Detailed analyses and discussions can be found in the \textbf{Appendix} \ref{appendix:exp_of_b_and_Q}.

\subsection{Empirical Setup}
\looseness = -1
We conduct experiments on the CIFAR-10 dataset \cite{krizhevsky2009learning} with a Dirichlet distribution (non-IID, $\alpha=0.3$) using LeNet to validate the excess error and generalization error of DSGD-MGS. To examine the impact of key hyperparameters, we follow the study by Hardt et al.\cite{hardt2016train} and investigate the weight distance ($\sum_{i=1}^n\sum_{j=1}^m ||\theta_j^{(t)} - \tilde{\theta}_j^{(t)}||_2^2$) and the loss distance ($R(\bar{\theta}^{(t)}) - R_S(\bar{\theta}^{(t)})$) when replacing only one data point in the training dataset. We primarily validate the experimental performance of key parameters in the DSGD-MGS algorithm, such as communication topology, the number of MGS steps, and the total number of clients. For fairness, when exploring one parameter, all other parameters are kept at the same settings. Further implementation details are provided in \textbf{Appendix} \ref{appendix_sec:detail of exp}.

\subsection{Experimental Validation of Generalization Error.}\label{sec:experiment of generalization error}

As shown in Figure \ref{fig:weight_diantance_and_loss_distance_cifar10}, subplots (a) and (b) respectively illustrate the weight distance and loss distance for different parameter settings of the DSGD-MGS algorithm on the perturbed dataset. Overall, both weight distance and loss distance exhibit the same power-law behavior as our theoretical bound $\mathcal{O}(T^{\frac{2c\beta}{2c\beta + 2}})$ (see Theorem \ref{the:generalization error of dsgd_mgs}).
Additionally, within each column of Figure \ref{fig:weight_diantance_and_loss_distance_cifar10} (corresponding to the same parameter setting), these two metrics follow similar trends, confirming the validity of Lemma \ref{lemma:ob-avg-gen} \cite{lei2020fine}, which states that the generalization error can indeed be captured by the stability bound.

\begin{figure*}[ht]
\vspace{-0.5em}
\begin{center}
\subfigure[Weight distance]{
    	\includegraphics[width=1\textwidth]{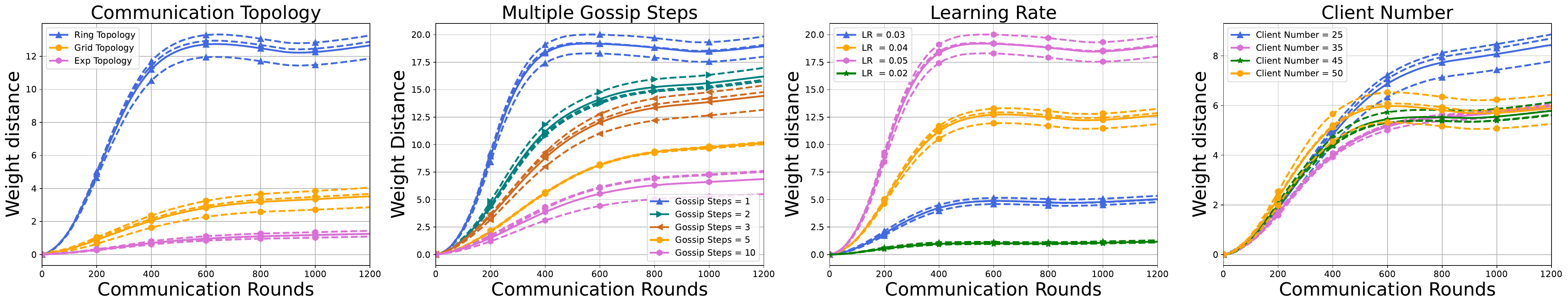}
    }
    \hfill
\subfigure[Loss distance]{
    	\includegraphics[width=1\textwidth]{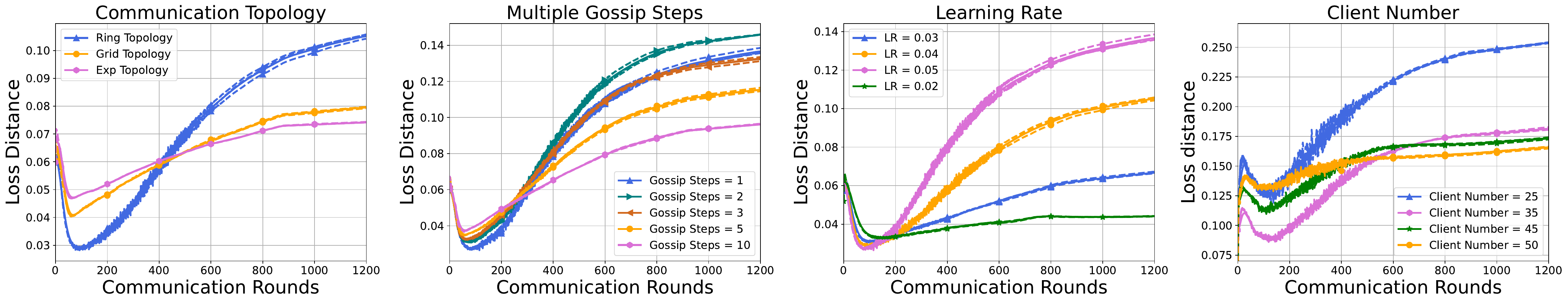}
    }
\end{center}
\vspace{-1em}
\caption{ A comparison of the $l_2$ weight distance and Loss distance (i.e. test loss - train loss) for the DSGD-MGS algorithm on the cifar10 dataset.}
\label{fig:weight_diantance_and_loss_distance_cifar10}
\vspace{-0.5em}
\end{figure*}

From subplots (a) and (b) in Figure \ref{fig:weight_diantance_and_loss_distance_cifar10}, we can observe the following patterns: 
1) \textit{\textbf{Using a communication topology with a smaller spectral gap}} (i.e., a larger $\rho$ in Theorem \ref{the:generalization error of dsgd_mgs}) leads to lower generalization error.
2) \textit{\textbf{Increasing the number of MGS effectively reduces the generalization error.}} For example, in terms of weight distance (Figure \ref{fig:weight_diantance_and_loss_distance_cifar10} (a)), setting $MGS=5$ reduces the weight distance to roughly half of that with $MGS=1$.
3) \textit{\textbf{Smaller learning rates help reduce generalization error}}, consistent with the findings in \cite{liuunderstanding} on decentralized federated learning.
4) \textit{\textbf{A larger client number (i.e., $m$ in Theorem \ref{the:generalization error of dsgd_mgs}) also helps reduce generalization error}}, reflecting a nearly linear speedup effect with the number of clients.
Notably, these observations align well with our theoretical results (see Theorem \ref{the:generalization error of dsgd_mgs} and Remark \ref{remark5:factor of gen error}).This further validates the correctness of our theoretical analysis.

%% file: tex/6.conclusion.tex
\section{Conclusion}

This paper is the first to establish the generalization error and excess error bounds for the DSGD-MGS algorithm in non-convex settings without the bounded gradients assumption. It addresses how MGS can exponentially reduce the generalization error bound and shows that even with a very large number of MGS steps, it cannot completely close the gap between decentralized and centralized training. Additionally, our theoretical results capture the impact of key factors like data heterogeneity $\delta$, communication topology spectrum $\xi$, Multiple Gossip Steps $Q$, client number $m$, and per-client data size $n$. Previous work has not unified the analysis of these critical parameters, and this paper fills that gap, offering both theoretical insights and experimental validation and significantly advancing the theoretical understanding of decentralized optimization.

\textbf{Limitation.} The theoretical findings in this paper depend on the properties of the last iteration of D-SGD in optimization theory, which is an emerging area yet to be explored. This paper derives the properties of the function value at the last iteration under the PL-condition. Future work can further explore the properties of the loss function gradient at the last iteration under non-convex conditions.

\section*{Acknowledgment}
Li Shen is supported by STI 2030—Major Projects (No. 2021ZD0201405), NSFC Grant (No.  62576364), Shenzhen Basic Research Project (Natural Science Foundation) Basic Research Key Project (NO. JCYJ20241202124430041). Miao Zhang is supported by the National Natural Science Foundation of China (NO. 62403484).

%% file: tex/7.proof_mgs.tex
\clearpage
\appendix
\part{Appendix}

\section{More Details about Experiments}

\subsection{Implementation Details for Experiments}\label{appendix_sec:detail of exp}
First, the perturbed dataset is constructed as follows: according to Definition \ref{def:on-average}, we randomly select a client, then randomly choose a data point from this client’s training set and swap it with a data point from the test set. This process creates the perturbed dataset. To enhance the robustness of our experimental results, all reported metrics are averaged over three independent runs with different random seeds.
Second, for the other experimental hyperparameters, aside from the experimental parameters to be explored, we use the following default settings: client number = 50, learning rate = 0.04, learning rate decay factor = 0.995, base communication topology = Ring, multiple gossip steps = 1, Dirichlet distribution coefficient (to control data heterogeneity) = 0.3, and communication rounds = 1200.


\subsection{Experimental Validation of Excess Error.}\label{sec:sec:experiment of excess error}

\begin{figure}[ht]
\begin{center}
\subfigure{
    	\includegraphics[width=0.8\textwidth]{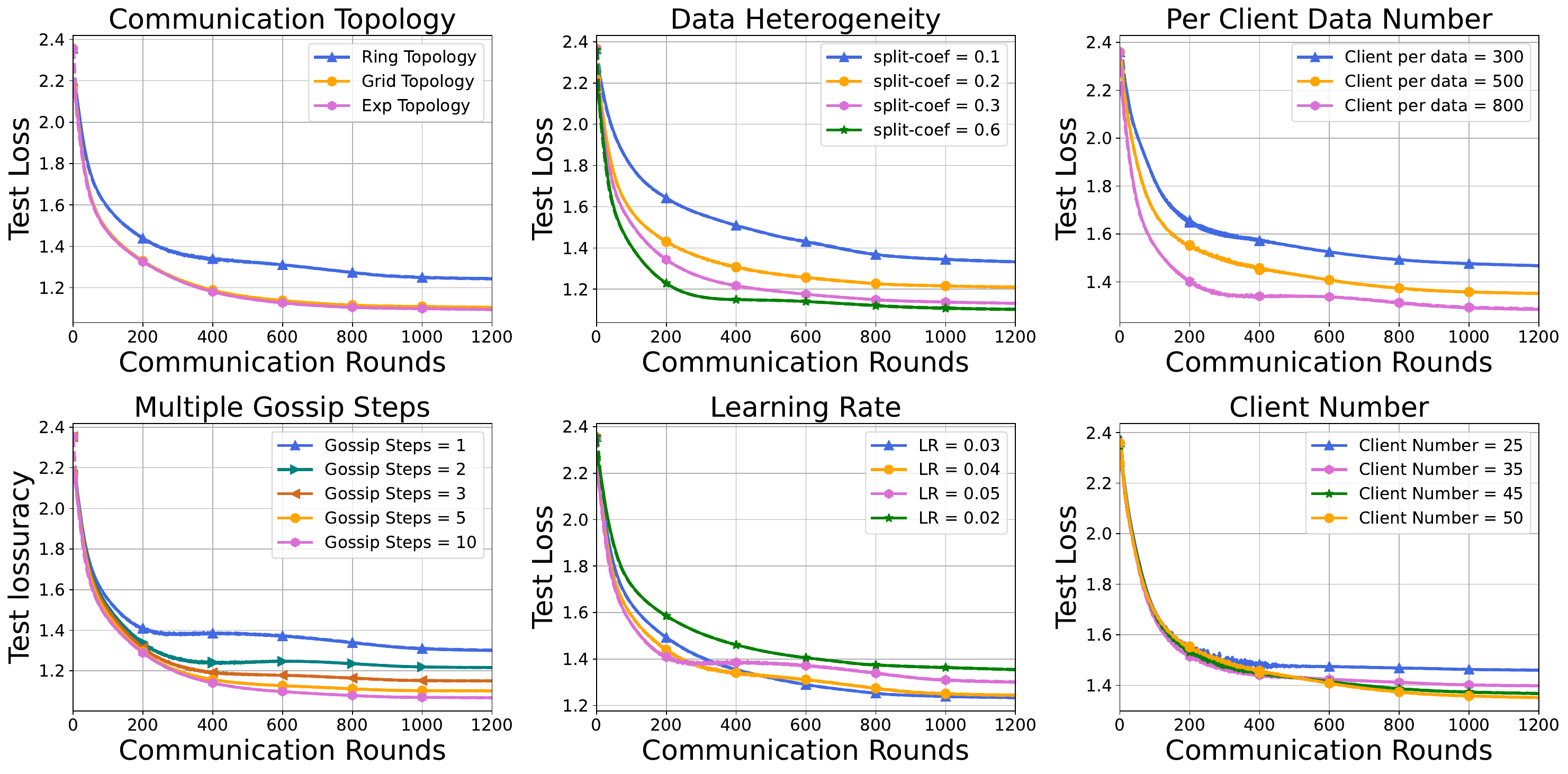}
    }
\end{center}
\caption{ \small The test loss of the DSGD-MGS algorithm on the cifar10 test dataset.}
\label{fig:test_loss_cifar10}
\end{figure}

According to the definition of excess error, $\mathbb{E}_{A,S} [R(A(S)) - R(\theta^\star)]$, since $\theta^\star$ represents the global optimal solution and is independent of the dataset $S$ (i.e., $\mathbb{E}_{A,S}[R(\theta^\star)]$ is a constant), the magnitude relationship of $\mathbb{E}_{A,S} [R(A(S)) - R(\theta^\star)]$ is equivalent to that of $\mathbb{E}_{A,S} [R(A(S))]$.
Therefore, the relative test errors $\mathbb{E}_{A,S} [R(A(S))]$ for different parameter settings can directly reflect the corresponding relationships in excess error.

Notably, our theoretical results provide, for the first time, excess error bounds for DSGD-MGS under non-convex assumptions and heterogeneous data, making them more general than the strongly convex results in \cite{ye2025icassp}. As shown in the "Data Heterogeneity" section of Figure \ref{fig:test_loss_cifar10}, the experimental findings align perfectly with our theoretical predictions, demonstrating that increasing data heterogeneity leads to higher excess error (see Theorem \ref{the:excess error} and Remark \ref{remark5:factor of gen error}).

Additionally, we conducted experiments with a fixed number of clients but varying per-client data sizes. As illustrated in the "Per Client Data Number" subplot of Figure \ref{fig:test_loss_cifar10}, increasing the amount of data per client (i.e., $n$ in Theorem \ref{the:excess error}) reduces excess error, exhibiting a nearly linear speedup, which is consistent with our theoretical analysis.
It is also worth noting the learning rate experiments. Comparing the "Learning Rate" results in Figure \ref{fig:test_loss_cifar10} and Figure \ref{fig:weight_diantance_and_loss_distance_cifar10}, we observe that generalization error decreases with smaller learning rates, while excess error shows a more nuanced trade-off, aligning well with our discussion in Remark \ref{remark:lr_difference} of Theorem \ref{the:excess error}.


\subsection{More Experiments Results of DSGD-MGS on CIFAR10}\label{appendix_sec:more figure on cifar10}

\begin{figure*}[ht]
\vspace{-0.5em}
\begin{center}
\subfigure[Weight distance]{
    	\includegraphics[width=0.9\textwidth]{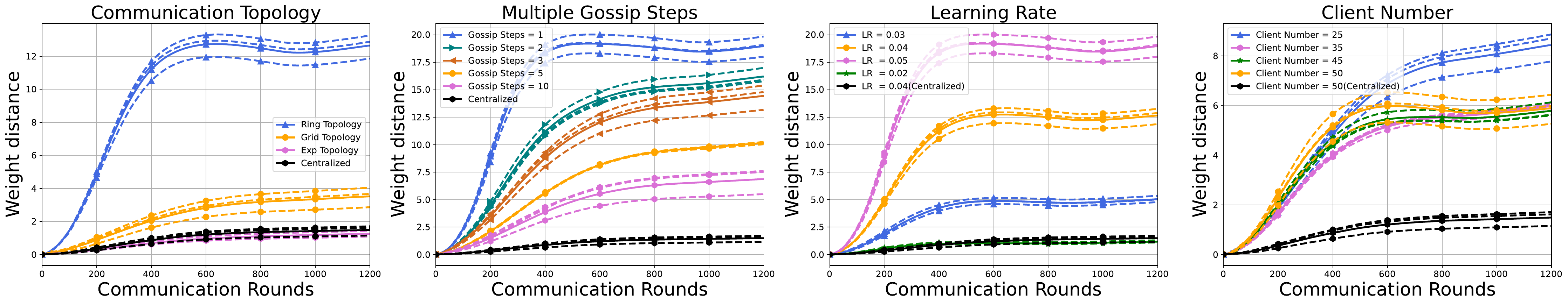}
    }
    \hfill
\subfigure[Loss distance]{
    	\includegraphics[width=0.9\textwidth]{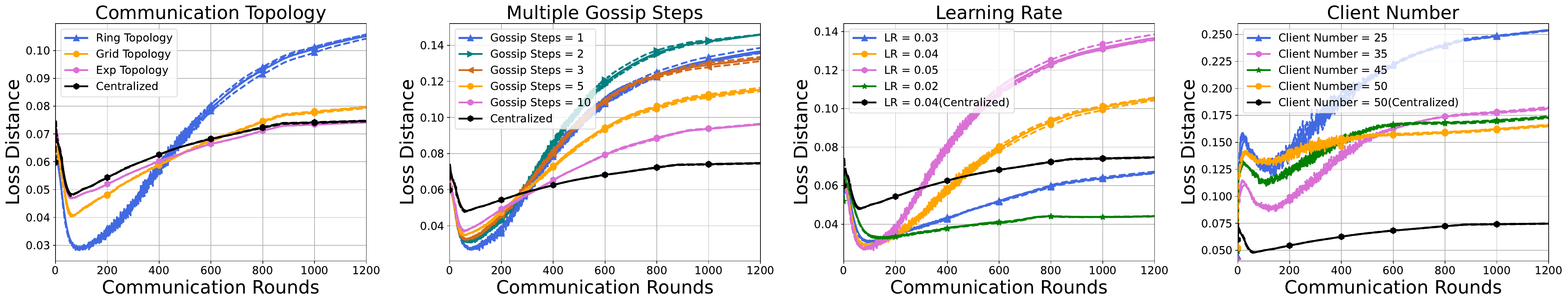}
    }
\end{center}
\vspace{-1em}
\caption{ A comparison of the $l_2$ weight distance and Loss distance (i.e. test loss - train loss) for the DSGD-MGS algorithm on the cifar10 dataset with centralized methods.}
\vspace{-0.5em}
\end{figure*}

\begin{figure}[ht]
\begin{center}
\subfigure{
    	\includegraphics[width=0.8\textwidth]{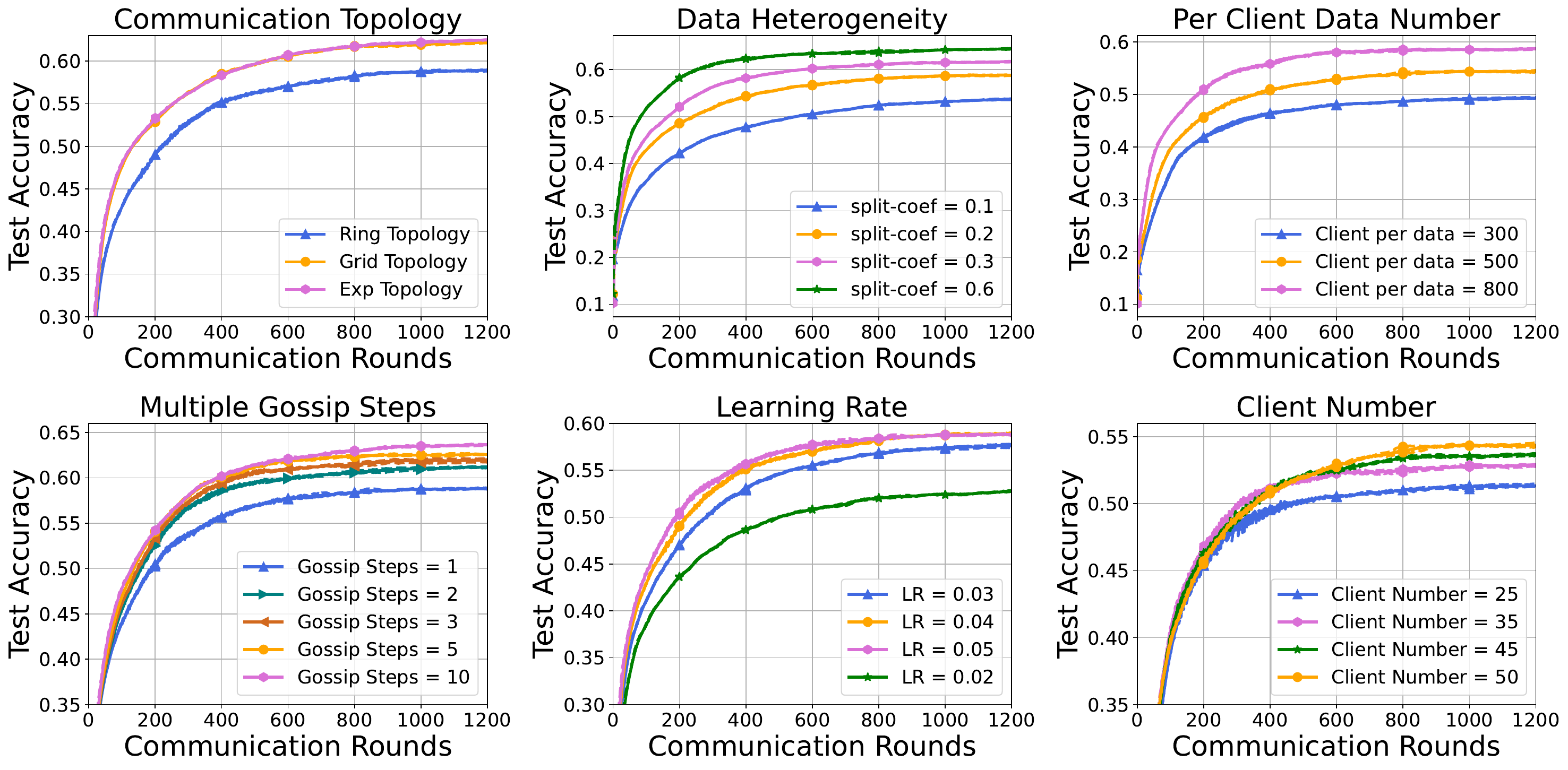}
    }
\end{center}
 \vspace{-0.4cm}
\caption{ \small The accuracy of the DSGD-MGS algorithm on the cifar10 test dataset.}
\label{fig:test_acc_cifar10}
 \vspace{-0.3cm}
\end{figure}

\begin{figure}[ht]
\begin{center}
\subfigure{
    	\includegraphics[width=0.8\textwidth]{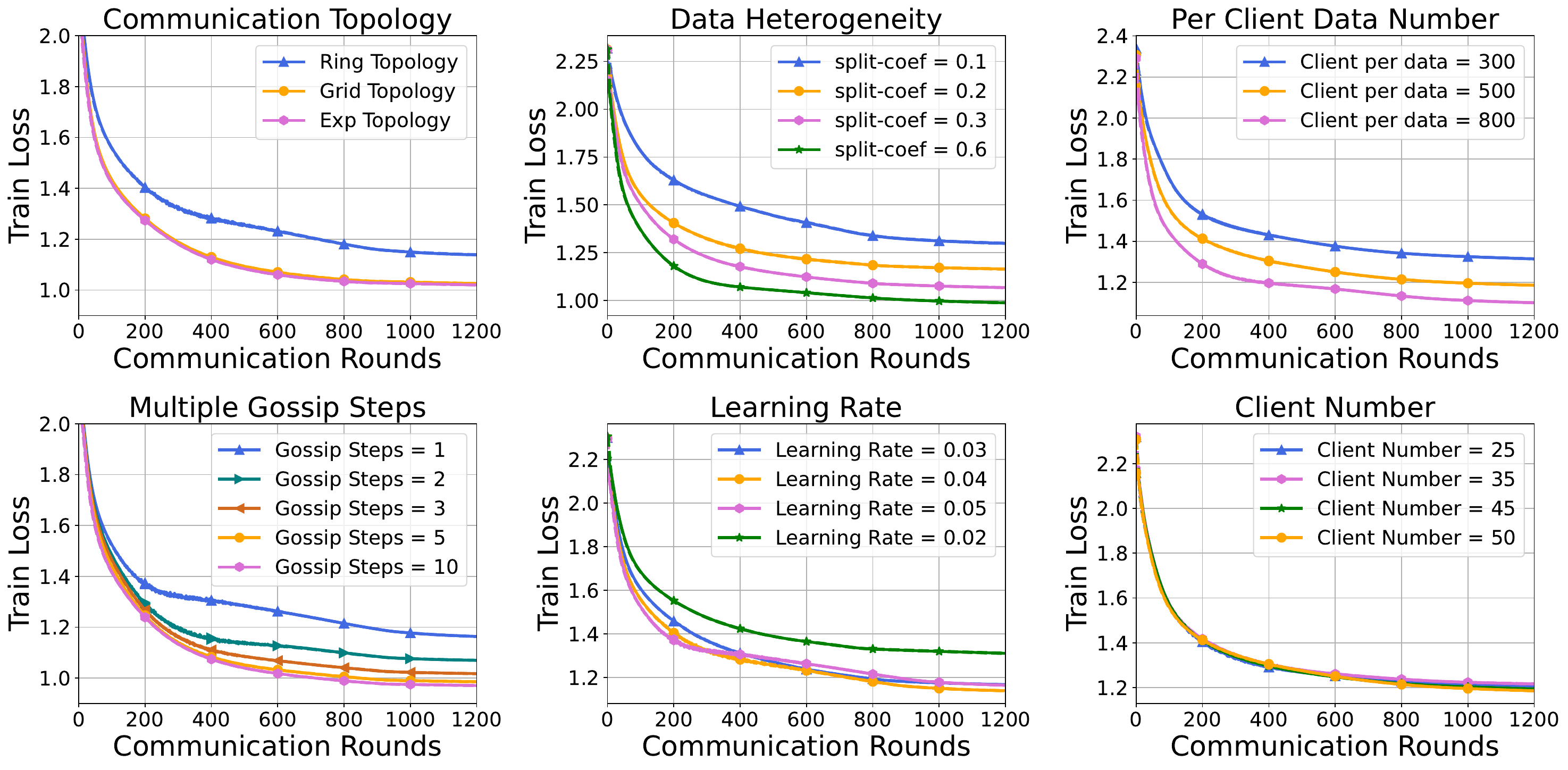}
    }
\end{center}
 \vspace{-0.4cm}
\caption{ \small The loss of the DSGD-MGS algorithm on the cifar10 train dataset.}
\label{fig:train_loss_cifar10}
 \vspace{-0.3cm}
\end{figure}

\section{Proof of generalization error for DSGD-MGS}\label{sec_appendix:proof of all}

We first present an important property of smooth functions, followed by the proof of Lemma \ref{lemma:ob-avg-gen}. Subsequently, we provide the proofs of other lemmas and theorems appearing in the main text.

\begin{lemma}\label{le:self-bound}
[Case of $\alpha = 1$ in \cite{lei2020fine}]\label{lemma:alpha1}
Assume that for all $z \in \mathcal{Z}$, the mapping $\mathbf{\theta} \mapsto \ell(\mathbf{\theta}; z)$ is non-negative, and its gradient $\mathbf{\theta} \mapsto \nabla \ell(\mathbf{\theta}; z)$ is $(1, \beta)$-Hölder continuous (Assumption \ref{ass:smooth}). Then there exists a constant
\(c_{1,1} = \sqrt{2\beta},\)
such that for all $\mathbf{\theta} \in \mathbb{R}^q$ and $z \in \mathcal{Z}$,
\[
\|\nabla \ell(\mathbf{\theta}; z)\|_2 \leq c_{1,1} \cdot \sqrt{\ell(\mathbf{\theta}; z)}.
\]
\end{lemma}
The above lemma \ref{le:self-bound} is also known as the self-bounding property of the function \( \ell \). Next, to prove Lemma \ref{lemma:ob-avg-gen}, we introduce a useful inequality for $\beta$-smooth functions $\ell$.
\begin{equation}\label{eq:beta-smooth}
    \ell(\theta;Z)\leq \ell(\tilde{\theta};Z)+\langle\theta-\tilde{\theta},\nabla \ell(\tilde{\theta};Z)\rangle+\frac{\beta\|\theta-\tilde{\theta}\|_2^2}{2}.
\end{equation}

\textbf{Proof of Lemma \ref{lemma:ob-avg-gen}.} Due to the symmetry, we know

\begin{equation}\label{eq:symmetry}
\begin{aligned}
\mathbb{E}_{S,A}\left[R(A(S))-R_S(A(S))\right] & =\mathbb{E}_{S,\widetilde{S},A}\left[\frac{1}{nm}\sum_{i,j}\left(R(A(S^{(ij)}))-R_{S}(A(S))\right)\right] \\
& =\mathbb{E}_{S,\widetilde{S},A}\left[\frac{1}{nm}\sum_{i,j}\left(\ell(A(S^{(ij)});Z_{ij})-\ell(A(S);Z_{ij})\right)\right]
\end{aligned}
\end{equation}

Since the loss function \( \ell \) satisfies \( \beta \)-smoothness (\ref{eq:beta-smooth}), we have:
\begin{equation}\label{eq:beta_smooth_mn}
\begin{aligned}
    \frac{1}{nm}\sum_{i,j}\ell(A(S^{(ij)});Z_{ij}) &\leq \frac{1}{nm}\sum_{i,j}\ell(A(S);Z_{ij}) + \frac{1}{nm}\sum_{i,j}\left<A(S^{(ij)}) - A(S), \nabla\ell(A(S);Z_{ij}) \right> \\
    & + \frac{\beta}{2nm}\sum_{i,j}\|A(S^{(ij)}) - A(S)\|^2
\end{aligned}
\end{equation}
That is 
\begin{equation*}
\begin{aligned}
    \mathbb{E}_{S,A}\left[R(A(S))-R_S(A(S))\right] \leq &  \frac{1}{nm}\sum_{i,j}\left<A(S^{(ij)}) - A(S), \nabla\ell(A(S);Z_{ij}) \right> \\
    & + \frac{\beta}{2nm}\sum_{i,j}\|A(S^{(ij)}) - A(S)\|^2
\end{aligned}
\end{equation*}
According to the Schwartz’s inequality we know
\begin{align*}
    \left<A(S^{(ij)}) - A(S), \nabla\ell(A(S);Z_{ij}) \right> &\leq \|A(S^{(ij)}) - A(S)\|_2 \|\nabla\ell(A(S);Z_{ij})\|_2\\
    & \leq \frac{\gamma}{2}\|A(S^{(ij)}) - A(S)\|_2^2 + \frac{1}{2\gamma}\|\nabla\ell(A(S);Z_{ij})\|^2
\end{align*}
Combining the above two inequalities together, we derive
\begin{align*}
    \mathbb{E}_{S,A}\left[R(A(S))-R_S(A(S))\right] &\leq \frac{1}{2mn\gamma}\sum_{i,j}\mathbb{E}_{A,S}[\|\nabla\ell(A(S);Z_{ij})\|^2] \\
    \quad & + \frac{\beta+\gamma}{2mn}\sum_{i,j}\mathbb{E}_{A,\Tilde{A},S}[\|A(S) - A(S^{(ij)})\|^2]
\end{align*}
Thus, we have completed the proof.

\subsection{Proof of Important Lemma}

Our analysis for the non-convex case relies on $l_2$ on-average model stability and leverages the fact that D-SGD can make
several steps before using the one example that has been swapped. This idea is summarized in the following lemma.

\begin{lemma} \label{lemma:key-non-conv} Assume that the loss function $\ell
(\cdot,z)$ is nonnegative for all $z$. For all $i=1,\ldots,n$ and $j=1,\ldots, m$, let $\{\theta_k^{(t)}\}_{t=0}^T$ and $ \{\tilde{\theta}_k^{(t)}(i,j)\}_{t=0}^T$, the iterates of agent $k = 1,\ldots, m$ for DSGD-MGS run on $S$ and $S^{(ij)}$ respectively. Then, for every $t_0 \in \{0, 1, \ldots, T\}$ we have:
\begin{align*}
    |\mathbb{E}_{A,S}[R(A_k(S)) - R_S(A_k(S))]| &\leq \frac{t_0}{n}\sup_{\theta, z} \ell(\theta;z)+ \frac{1}{2mn\gamma}\sum_{i=1}^n\sum_{j=1}^m\mathbb{E}[\|\nabla\ell(A_k(S);Z_{ij})\|^2]  \\
    &\quad + \frac{\gamma + \beta}{2mn}\sum_{i=1}^n\sum_{j=1}^m\mathbb{E}[\delta_k^{(T)}(i,j) \big| \delta^{(t_0)}(i,j) = \mathbf{0}]
\end{align*}
where $\delta^{(t)}(i,j)$ is the vector containing $\forall k=1,\ldots,m$, $\delta_k^{(t)}(i,j) = \|\theta_k^{(t)} - \tilde{\theta}_k^{(t)}(i,j)\|_2^2$.
\end{lemma}

\begin{proof}
    Consider the notation of Def. \ref{def:on-average} and notice that 

$$R(A_k(S)) = \frac{1}{m}\sum_{j=1}^m\mathbb{E}_{Z \sim \mathcal{D}_j} [\ell(A_k(S); Z)] = \frac{1}{mn}\sum_{j=1}^m\sum_{j=1}^n\mathbb{E}_{\Tilde{S}} [\ell(A_k(S); \Tilde{Z}_{ij})].$$ Then, for all $k = 1,\ldots,m$, by linearity of expectation we have

\begin{align}
    \mathbb{E}_{A,S}[R(A_k(S)) - R_S(A_k(S))] & = \mathbb{E}_{A,S,\Tilde{S}}\Bigg[\frac{1}{mn}\sum_{j=1}^m\sum_{i=1}^n \Big(\ell(A_k(S); \Tilde{Z}_{ij}) - \ell(A_k(S); Z_{ij})\Big)\Bigg] \nonumber \\
    & = \mathbb{E}_{A,S,\Tilde{S}}\Bigg[\frac{1}{mn}\sum_{j=1}^m\sum_{i=1}^n \Big(\ell(A_k(S^{(ij)}); Z_{ij}) - \ell(A_k(S); Z_{ij})\Big)\Bigg].\nonumber
\end{align}

Hence, 
\begin{align*}
    |\mathbb{E}_{A,S}[R(A_k(S)) - R_S(A_k(S))]| & \leq \mathbb{E}_{A,S,\Tilde{S}}\Bigg[\frac{1}{mn}\sum_{j=1}^m\sum_{i=1}^n \Big|\ell(A_k(S^{(ij)}); Z_{ij}) - \ell(A_k(S); Z_{ij})\Big|\Bigg]  \\
    & = \frac{1}{mn}\sum_{j=1}^m\sum_{i=1}^n \mathbb{E}_{A,S,\Tilde{S}}\Bigg[ \Big|\ell(A_k(S^{(ij)}); Z_{ij}) - \ell(A_k(S); Z_{ij})\Big|\Bigg]
\end{align*}

Let the event $\mathcal{E} (i,j) = \{\delta^{(t_0)}(i,j) = \mathbf{0}\} $, we have $\forall i,j$:
\begin{equation}\label{eq:nonconvex_lemma}
\begin{aligned}
    \mathbb{E}_{A,S,\Tilde{S}}& \Big[ \big|\ell(A_k(S^{(ij)}); Z_{ij})  - \ell(A_k(S); Z_{ij})\big|\Big] \\
    & = \mathbb{P}(\mathcal{E}(i,j))\mathbb{E}[|\ell(A_k(S^{(ij)}); Z_{ij}) - \ell(A_k(S); Z_{ij})| \big| \mathcal{E}(i,j)]  \\
    & \hspace{2cm}+  \mathbb{P}(\mathcal{E}(i,j)^c)\mathbb{E}[|\ell(A_k(S^{(ij)}); Z_{ij}) - \ell(A_k(S); Z_{ij})| \big| \mathcal{E}(i,j)^c] \\
    & \leq \mathbb{E}[|\ell(A_k(S^{(ij)}); Z_{ij}) - \ell(A_k(S); Z_{ij})| \big| \mathcal{E}(i,j)] +\mathbb{P}(\mathcal{E}(i,j)^c)\cdot \sup_{\theta,z}\ell(\theta;z) 
\end{aligned}
\end{equation}
Considering the smoothness of $\ell$, we have:
\begin{align}\label{eq:beta_smooth_single}
    \ell(A_k(S^{(ij)}) &;Z_{ij}) - \ell(A_k(S);Z_{ij}) \\  \nonumber
    &\leq   \left<A_k(S^{(ij)}) - A_k(S), \nabla\ell(A_k(S);Z_{ij}) \right>  + \frac{\beta}{2}\|A_k(S^{(ij)}) - A_k(S)\|^2 \\ \nonumber
    &\leq \frac{\gamma}{2}\|A_k(S^{(ij)}) - A_k(S)\|^2 + \frac{1}{2\gamma}\|\nabla\ell(A_k(S);Z_{ij})\|^2 + \frac{\beta}{2}\|A_k(S^{(ij)}) - A_k(S)\|^2 \\ \nonumber
    & = \frac{\gamma + \beta}{2}\|A_k(S^{(ij)}) - A_k(S)\|^2 + \frac{1}{2\gamma}\|\nabla\ell(A_k(S);Z_{ij})\|^2 
\end{align}
Where the second inequality uses the bound  
\(
\langle a, b \rangle \leq \frac{\gamma}{2} \|a\|^2 + \frac{1}{2\gamma} \|b\|^2,
\)
which holds for any \( \gamma > 0 \). Combining the above inequality (\ref{eq:beta_smooth_single}) with inequality (\ref{eq:nonconvex_lemma}), we obtain:
\begin{equation}\label{eq:nonconvex_lemma_final}
\begin{aligned}
    &\mathbb{E}_{A,S,\Tilde{S}} \Big[ \big|\ell(A_k(S^{(ij)}); Z_{ij})  - \ell(A_k(S); Z_{ij})\big|\Big] \\
    & \leq \mathbb{E}[|\ell(A_k(S^{(ij)}); Z_{ij}) - \ell(A_k(S); Z_{ij})| \big| \mathcal{E}(i,j)] +\mathbb{P}(\mathcal{E}(i,j)^c)\cdot \sup_{\theta,z}\ell(\theta;z) \\ 
    & \leq \frac{\gamma + \beta}{2}\mathbb{E}[\|A_k(S)  - A_k(S^{(ij)})\|^2 \big| \mathcal{E}(i,j)] \\
    & \quad + \frac{1}{2\gamma} \mathbb{E}[\|\nabla\ell(A_k(S);Z_{ij})\|^2\big| \mathcal{E}(i,j)]  +\mathbb{P}(\mathcal{E}(i,j)^c)\cdot \sup_{\theta,z}\ell(\theta;z) \\
    & = \frac{\gamma + \beta}{2}\mathbb{E}[\delta_k^{(T)}(i,j) \big| \mathcal{E}(i,j)] + \frac{1}{2\gamma} \mathbb{E}[\|\nabla\ell(A_k(S);Z_{ij})\|^2]  + \mathbb{P}(\mathcal{E}(i,j)^c)\cdot \sup_{\theta,z}\ell(\theta;z)
\end{aligned}
\end{equation}
The last equality follows from the independence between \( \|\nabla\ell(A_k(S);Z_{ij})\|^2 \) and \( \mathcal{E}(i,j) \).
It remains to bound $\mathbb{P}(\mathcal{E}(i,j)^c)$. Let $T_0$ be the random variable of the first time step DSGD-MGS uses the swapped example. Since we necessarily have $\{T_0 > t_0\} \subset \mathcal{E}(i,j)$, we have $ \mathcal{E}(i,j)^c \subset \{T_0 \leq t_0\}$ and therefore $\mathbb{P}(\mathcal{E}(i,j)^c) \leq \mathbb{P}(T_0 \leq t_0) = \sum_{t=1}^{t_0}\mathbb{P}(T_0=t)\leq \sum_{t=1}^{t_0}\frac{1}{n} = \frac{t_0}{n}$. Averaging over $i$ and $j$ completes the proof.

\end{proof}

We can now move on to the proof of the main theorem. We first apply Lemma \ref{lemma:key-non-conv} and the fact that, by assumption, $\ell \in [0,1]$, so that for any $t_0 \in \{0, 1, \ldots, T\}$ and any $k=1,\ldots, m$, we have:
\begin{equation}\label{eq:right-hand}
\begin{aligned}
    |\mathbb{E}_{A,S}[R(A_k(S)) - R_S(A_k(S))]| &\leq \frac{t_0}{n} + \frac{1}{2mn\gamma}\sum_{i,j}\mathbb{E}[\|\nabla\ell(A_k(S);Z_{ij})\|^2]  \\
    &\quad + \frac{\gamma + \beta}{2mn}\sum_{i,j}\mathbb{E}[\delta_k^{(T)}(i,j) \big| \delta^{(t_0)}(i,j) = \mathbf{0}]
\end{aligned}
\end{equation}
It remains to control the right-hand term of Equation \eqref{eq:right-hand}. We start with the proof for DSGD-MGS. 

\subsection{Proof of \texorpdfstring{$l_2$}{l2} on average model stability of DSGD-MGS.}

For a fixed couple $(i,j)$, we are first going to control the vector $\Delta^{(t)}(i,j) \triangleq\mathbb{E}[\delta^{(t)}(i,j) | \delta^{(t_0)}(i,j) = \mathbf{0}]$, where $\delta^{(t)}(i,j)$ is the vector containing $\forall k=1,\ldots,m$, $\delta_k^{(t)}(i,j) = \|\theta_k^{(t)} - \tilde{\theta}_k^{(t)}(i,j)\|_2^2$. When it is clear from context, we simply write $\tilde{\theta}_k^{(t)}(i,j) = \tilde{\theta}_k^{(t)}$. 

We first estimate \( \|\theta_k^{(t+1)} - \tilde{\theta}_k^{(t+1)}\|_2^2 \).
\begin{equation}\label{eq:recur_k_nonconvex}
\begin{aligned}
    \|\theta_k^{(t+1)} - \tilde{\theta}_k^{(t+1)}\|_2^2 & = 
    \left\| \sum_{l=1}^mW_{kl}\left[ \theta_l^{(t)}-\tilde{\theta}_l^{(t)} +\eta_t\left(  \nabla\ell(\tilde{\theta}_l^{(t)};Z'_{I_l^tl}) - \nabla\ell(\theta_l^{(t)};Z_{I_l^tl})  \right)     \right]\right\|^2 \\
    &\leq \sum_{l=1}^mW_{kl} \left\| \theta_l^{(t)}-\tilde{\theta}_l^{(t)} +\eta_t\left(  \nabla\ell(\tilde{\theta}_l^{(t)};Z_{I_l^tl}) - \nabla\ell(\theta_l^{(t)};Z_{I_l^tl})  \right)   \right\|^2\\
    &\leq \sum_{l\neq j}^mW_{kl} \left\| \theta_l^{(t)}-\tilde{\theta}_l^{(t)} +\eta_t\left(  \nabla\ell(\tilde{\theta}_l^{(t)};Z_{I_l^tl}) - \nabla\ell(\theta_l^{(t)};Z_{I_l^tl})  \right)   \right\|^2 \\
    & \quad + W_{kj}\left\| \theta_j^{(t)}-\tilde{\theta}_j^{(t)} +\eta_t\left(  \nabla\ell(\tilde{\theta}_j^{(t)};Z_{I_j^tj}) - \nabla\ell(\theta_j^{(t)};Z_{I_j^tj})  \right)   \right\|^2\\
    &\leq (1+\eta_t\beta)^2\sum_{l\neq j}^mW_{kl}\|\theta_k^{(t)} - \tilde{\theta}_k^{(t)}\|_2^2 \\
    & \quad + W_{kj}\left\| \theta_j^{(t)}-\tilde{\theta}_j^{(t)} +\eta_t\left(  \nabla\ell(\tilde{\theta}_j^{(t)};Z_{I_j^tj}) - \nabla\ell(\theta_j^{(t)};Z_{I_j^tj})  \right)   \right\|^2
\end{aligned}
\end{equation}
The first inequality in the above expression follows from Jensen's inequality, and the last inequality follows from the \((1 + \eta_t \beta)\)-expansiveness of \( \ell \) \cite{hardt2016train} when $l \neq j$. Next, we perform a analysis of the second term on the right-hand side of the above inequality.

With probability $1-\frac{1}{n}$, $I^t_j\neq i$ so $Z_{I^t_jj} = Z'_{I^t_jj}$. We have
\begin{equation}
    \begin{aligned}
        \left\| \theta_j^{(t)}-\tilde{\theta}_j^{(t)} +\eta_t\left(  \nabla\ell(\tilde{\theta}_j^{(t)};Z_{I_j^tj}) - \nabla\ell(\theta_j^{(t)};Z_{I_j^tj})  \right)   \right\|^2 \leq (1+\eta_t\beta)^2\|\theta_j^{(t)}-\tilde{\theta}_j^{(t)}\|^2
    \end{aligned}
\end{equation}

With probability $\frac{1}{n}$, $I^t_j=i$ and in that case $Z_{I^t_jj} = Z_{ij}\neq \Tilde{Z}_{ij} =  Z'_{I^t_jj} $. 
\begin{equation}
    \begin{aligned}
       &\left\| \theta_j^{(t)}-\tilde{\theta}_j^{(t)} + \eta_t\left(  \nabla\ell(\tilde{\theta}_j^{(t)};Z'_{ij}) - \nabla\ell(\theta_j^{(t)};Z_{ij})  \right)   \right\|^2  \leq
       (1+p)\|\theta_j^{(t)}-\tilde{\theta}_j^{(t)}\|^2 \\
       &\hspace{3cm} + 2\eta_t^2(1+ p^{-1})\|\nabla\ell(\tilde{\theta}_j^{(t)};Z'_{ij})\|^2 + 2\eta_t^2(1+p^{-1})\|\nabla\ell(\theta_j^{(t)};Z_{ij})\|^2
    \end{aligned}
\end{equation}
Considering that \( I^t_k \) follows a uniform distribution ($I^t_k\sim \mathcal{U}\{1,\ldots, n\}$), we get
\begin{align}\label{eq:l=j}
    &\left\| \theta_j^{(t)}-\tilde{\theta}_j^{(t)} + \eta_t\left(  \nabla\ell(\tilde{\theta}_j^{(t)};Z'_{ij}) - \nabla\ell(\theta_j^{(t)};Z_{ij})  \right)   \right\|^2 \leq (1+\frac{p}{n})(1+\eta_t\beta)^2\left\| \theta_j^{(t)}-\tilde{\theta}_j^{(t)}\right\|^2 \\ \nonumber
    & \hspace{3cm} + \frac{2\eta_t^2(1+ p^{-1})}{n}\|\nabla\ell(\tilde{\theta}_j^{(t)};Z'_{ij})\|^2 + \frac{2\eta_t^2(1+p^{-1})}{n} \|\nabla\ell(\theta_j^{(t)};Z_{ij})\|^2
\end{align}
Substituting equation (\ref{eq:l=j}) into equation (\ref{eq:recur_k_nonconvex}), we obtain:
\begin{align}
    \|\theta_k^{(t+1)} - \tilde{\theta}_k^{(t+1)}\|_2^2 & \leq (1+\frac{p}{n})(1+\eta_t\beta)^2\sum_{l=1}^mW_{kl}\|\theta_k^{(t)} - \tilde{\theta}_k^{(t)}\|_2^2 \\ \nonumber
    & \quad + \frac{2\eta_t^2(1+ p^{-1})}{n}W_{kj}\left(\|\nabla\ell(\tilde{\theta}_j^{(t)};Z'_{ij})\|^2 + \|\nabla\ell(\theta_j^{(t)};Z_{ij})\|^2\right)
\end{align}
Given that  \(\mathbb{E}_{S,\Tilde{S},A}\left[\|\nabla\ell(\theta_j^{(t)};Z_{ij})\|^2\right] = \mathbb{E}_{S,\Tilde{S},A}\left[\|\nabla\ell(\tilde{\theta}_j^{(t)};Z'_{ij})\|^2\right],\)
and to simplify the notation, we denote \(\mathbb{E}_{S,\Tilde{S},A}[\cdot] = \mathbb{E}[\cdot]\), we then obtain the following: 
\begin{align}
    \mathbb{E}[\|\theta_k^{(t+1)} - \tilde{\theta}_k^{(t+1)}\|_2^2] & \leq (1+\frac{p}{n})(1+\eta_t\beta)^2\sum_{l=1}^mW_{kl}\mathbb{E}[\|\theta_k^{(t)} - \tilde{\theta}_k^{(t)}\|_2^2] \\ \nonumber
    & \quad + \frac{4\eta_t^2(1+ p^{-1})}{n}W_{kj} \mathbb{E}[\|\nabla\ell(\theta_j^{(t)};Z_{ij})\|^2]
\end{align}
From the previous equations and let the vector \( G^{(t)} \in \mathbb{R}^m \) be defined such that its \( j \)-th component is  
\(
G_j^{(t)} = \mathbb{E}[\|\nabla\ell(\theta_j^{(t)};Z_{ij})\|^2]
\), we get that $\Delta^{(t+1)}(i,j) \le (1+\frac{p}{n})(1+\eta_t\beta)^2W\Delta^{(t)}(i,j) + \frac{4\eta_t^2(1+ p^{-1})}{n}W_{j}\circ G^{(t)}$ (the inequality, and the following ones are meant coordinate-wise), where \( {W}_j \) denotes the \( j \)-th column of matrix \( W \), and \( \circ \) represents the Hadamard product . Let $\Delta^{(t)}=\frac{1}{mn}\sum_{i,j}\Delta^{(t)}(i,j)$, then using the fact that $\eta_t\leq \frac{c}{t+1}$, $c>0$, we have $\forall t\geq t_0$:
\begin{equation}
    \begin{aligned}
        \Delta^{(t+1)} &\leq (1+\eta_t\beta)^2(1+\frac{p}{n})W\Delta^{(t)} + \frac{4\eta_t^2(1+ p^{-1})}{nm}\sum_{j=1}^mW_{j}\circ G^{(t)}\\
        & = (1+\eta_t\beta)^2(1+\frac{p}{n})W\Delta^{(t)} + \frac{4\eta_t^2(1+ p^{-1})}{nm}G^{(t)}\\
        & \leq (1+\frac{c\beta}{t+1})^2(1+\frac{p}{n})W\Delta^{(t)} + \frac{4(1+ p^{-1})}{nm}\frac{c^2}{(t+1)^2}G^{(t)}\\
    \end{aligned}
\end{equation}
Using Lemma \ref{le:self-bound} and the assumption that \( \ell \in [0,1] \), we can derive  
\(G^{(t)} \leq \sqrt{2\beta} \, \mathbf{1},\)
where \( \mathbf{1} \) is the all-ones vector. Then 
\begin{equation}
    \begin{aligned}
        \Delta^{(t+1)}
        & \leq (1+\frac{c\beta}{t+1})^2(1+\frac{p}{n})W\Delta^{(t)} + \frac{4\sqrt{2\beta}(1+ p^{-1})}{nm}\frac{c^2}{(t+1)^2}\mathbf{1}\\
    \end{aligned}
\end{equation}

Since $\Delta^{(t_0)} = \mathbf{0}$, we can unroll the previous recursion from $T$ to $t_0+1$ and get:
\begin{equation}
\begin{aligned}
\Delta^{(T)} &\le  \sum_{s=t_0}^{T-1} \left( \prod_{k=s+1}^{T-1} \left(1+\frac{c\beta}{k+1}\right)^2 \left(1+\frac{p}{n}\right) W \right) \cdot \frac{4\sqrt{2\beta}c^2(1+p^{-1})}{nm(s+1)^2} \mathbf{1}.
\end{aligned}
\end{equation}
Then, we focus on the coordinate of interest $k$ and using the fact that $1+x\leq \exp(x)$, we have:
\begin{equation}
\begin{aligned}
\Delta_k^{(T)} &\le  \sum_{s=t_0}^{T-1} \left( \prod_{k=s+1}^{T-1} \exp(\frac{2c\beta}{k+1}) \left(1+\frac{p}{n}\right) \right) \cdot \frac{4\sqrt{2\beta}c^2(1+p^{-1})}{nm(s+1)^2} .\\
&\le  \sum_{s=t_0}^{T-1} \left( \left(1+\frac{p}{n}\right)^{T-s-1} W^{T-s-1} \exp(2c\beta\sum_{k=s+1}^{T-1}\frac{1}{k+1})  \right) \cdot \frac{4\sqrt{2\beta}c^2(1+p^{-1})}{nm(s+1)^2} .\\
&\le  \sum_{s=t_0}^{T-1} \left( \left(1+\frac{p}{n}\right)^{T-s-1} W^{T-s-1} \exp(2c\beta\log(\frac{T}{s+1}))  \right) \cdot \frac{4\sqrt{2\beta}c^2(1+p^{-1})}{nm(s+1)^2} .\\
& =  \sum_{s=t_0}^{T-1} \left( \left(1+\frac{p}{n}\right)^{T-s-1} W^{T-s-1} \left(\frac{T}{s+1}\right)^{2c\beta}  \right) \cdot \frac{4\sqrt{2\beta}c^2(1+p^{-1})}{nm(s+1)^2} .\\
& =  \sum_{s=t_0}^{T-1} \left( \left(1+\frac{p}{n}\right)^{T-s-1} \left(\frac{T}{s+1}\right)^{2c\beta}  \right) \cdot \frac{4\sqrt{2\beta}c^2(1+p^{-1})}{nm(s+1)^2} .\\
\end{aligned}
\end{equation}
Let \( p = \frac{n}{T - t_0 - 1} > 1 \), then for \( s \geq t_0 \), we have  
\(\left(1 + \frac{p}{n}\right)^{T - s - 1} < \left(1 + \frac{1}{T - t_0 - 1}\right)^{T - t_0 - 1} < e,\)
and also  \(1 + p^{-1} < 2\), where $e$ is euler's number. Then, we have
\begin{align}\label{eq:matrix_l2_stability}
    \Delta_k^{(T)} & \le \sum_{s=t_0}^{T-1} \left(\frac{T}{s+1}\right)^{2c\beta}  \cdot \frac{8e\sqrt{2\beta}c^2}{nm(s+1)^2} \\ \nonumber
    &\le  \frac{8e\sqrt{2\beta}c^2T^{2c\beta}}{nm} \cdot\int_{t_0}^{T-1}s^{-2c\beta -2}\,ds \\ \nonumber
    &\le  \frac{8e\sqrt{2\beta}c^2}{(1+2c\beta)nmt_0} \left(\frac{T}{t_0}\right)^{2c\beta}
\end{align}
We then derive the component-wise form of inequality (\ref{eq:matrix_l2_stability}).
\begin{align}\label{eq:l2_stability_single}
    \frac{1}{mn}\sum_{i=1}^n\sum_{j=1}^m\mathbb{E}[\delta_k^{(T)}(i,j) \big| \delta^{(t_0)}(i,j) =\mathbf{0}] \leq \frac{8e\sqrt{2\beta}c^2}{(1+2c\beta)nmt_0} \left(\frac{T}{t_0}\right)^{2c\beta}
\end{align}

\subsection{Proof of generalization of DSGD-MGS}

By substituting (\ref{eq:l2_stability_single}) into (\ref{eq:right-hand}), we obtain the following.
\begin{equation}\label{eq:right-hand-int}
\begin{aligned}
    |\mathbb{E}_{A,S}[R(A_k(S)) - R_S(A_k(S))]| &\leq \frac{t_0}{n} + \frac{1}{2mn\gamma}\sum_{i,j}\mathbb{E}[\|\nabla\ell(A_k(S);Z_{ij})\|^2]  \\
    &\quad + \frac{\gamma + \beta}{2}\frac{8e\sqrt{2\beta}c^2}{(1+2c\beta)nmt_0} \left(\frac{T}{t_0}\right)^{2c\beta}
\end{aligned}
\end{equation}
Treating equation (\ref{eq:right-hand-int}) as a function of \( t_0 \), and noting that the left-hand side is independent of \( t_0 \), equation (\ref{eq:right-hand-int}) holds for any \( t_0 \). Without loss of generality, let $t_0=\left(\frac{4(\gamma+\beta)e\sqrt{2\beta}c^{2T^{2c\beta}}}{m}\right)^{\frac{1}{2c\beta+2}}$, yielding the following expression.
\begin{equation}\label{eq:optimal_t_0}
\begin{aligned}
    |\mathbb{E}_{A,S}[R(A_k(S)) - R_S(A_k(S))]| &\leq  \frac{1}{2mn\gamma}\sum_{i,j}\mathbb{E}[\|\nabla\ell(A_k(S);Z_{ij})\|^2] \\
    &\quad + \frac{2(c\beta+1)}{n(2c\beta+1)}\cdot\left(\frac{4(\gamma+\beta)e\sqrt{2\beta}c^2T^{2c\beta}}{m}\right)^{\frac{1}{2c\beta+2}}
\end{aligned}
\end{equation}
Let \( f(\gamma) \) denote the right-hand side of equation (\ref{eq:optimal_t_0}). We proceed to analyze the approximate minimum of \( f(\gamma) \). Let $C_1 = \frac{1}{2mn} \sum_{i,j}\mathbb{E}[\|\nabla\ell(A_k(S);Z_{ij})\|^2], C_2 = \frac{2(c\beta+1)}{n(2c\beta+1)} \left( \frac{4e\sqrt{2\beta}c^2T^{2c\beta}}{m} \right)^{\frac{1}{2c\beta+2}}$, and $\alpha = \frac{1}{2c\beta+2}$, 
We aim to find an upper bound for the minimum value of the function $f(\gamma)$ defined as:
$$f(\gamma) = C_1\gamma^{-1} + C_2(\gamma+\beta)^{\alpha}$$
where $\gamma > 0$, $\beta > 0$, 
We assume $c \ge 1$ and $\beta > 0$, which implies $2c\beta+2 > 2$, and thus $0 < \alpha < 1/2$.

Finding the exact minimum of $f(\gamma)$ requires solving $f'(\gamma) = -C_1\gamma^{-2} + \alpha C_2 (\gamma+\beta)^{\alpha-1} = 0$, which yields the equation $\frac{\gamma^2}{(\gamma+\beta)^{1-\alpha}} = \frac{C_1}{\alpha C_2}$. This equation is generally intractable to solve analytically for $\gamma$. Therefore, it is not amenable to analysis, and we need to approximate \( f(\gamma) \) to enable an explicit analysis of the upper bound on the generalization error. Next, we employ inequalities to derive an analytically tractable approximation of the generalization bound.

\textbf{Seeking an analytically tractable approximation of the generalization bound:}

We seek an analytically tractable upper bound for the minimum value, $\min_{\gamma > 0} f(\gamma)$. We utilize the standard inequality $(x+y)^p \le x^p + y^p$ which holds for $x, y > 0$ and $0 < p < 1$. Since $0 < \alpha < 1$, we can apply this inequality to the term $(\gamma+\beta)^{\alpha}$:
$$ (\gamma+\beta)^{\alpha} \le \gamma^{\alpha} + \beta^{\alpha} $$
Substituting this into the expression for $f(\gamma)$ yields an upper bound:
$$ f(\gamma) \le C_1\gamma^{-1} + C_2(\gamma^{\alpha} + \beta^{\alpha}) $$
Let $g(\gamma) = C_1\gamma^{-1} + C_2\gamma^{\alpha} + C_2\beta^{\alpha}$. The minimum of $f(\gamma)$ is bounded by the minimum of $g(\gamma)$:
$$ \min_{\gamma > 0} f(\gamma) \le \min_{\gamma > 0} g(\gamma) $$
We find the minimum of $g(\gamma)$ by setting its derivative with respect to $\gamma$ to zero:
$$ g'(\gamma) = \frac{d}{d\gamma} (C_1\gamma^{-1} + C_2\gamma^{\alpha} + C_2\beta^{\alpha}) = -C_1\gamma^{-2} + \alpha C_2 \gamma^{\alpha-1} $$
Setting $g'(\gamma) = 0$:
$$ C_1\gamma^{-2} = \alpha C_2 \gamma^{\alpha-1} $$
$$ \gamma^{\alpha+1} = \frac{C_1}{\alpha C_2} $$
The minimizer $\tilde{\gamma}^{*}$ for $g(\gamma)$ is:
$$ \tilde{\gamma}^{*} = \left( \frac{C_1}{\alpha C_2} \right)^{\frac{1}{\alpha+1}} $$
Substituting $\tilde{\gamma}^{*}$ back into $g(\gamma)$ gives the minimum value of $g(\gamma)$:
\begin{align*} \min_{\gamma > 0} g(\gamma) = g(\tilde{\gamma}^{*}) &= C_1(\tilde{\gamma}^{*})^{-1} + C_2(\tilde{\gamma}^{*})^{\alpha} + C_2\beta^{\alpha} \\ &= C_1 \left( \frac{C_1}{\alpha C_2} \right)^{\frac{-1}{\alpha+1}} + C_2 \left( \frac{C_1}{\alpha C_2} \right)^{\frac{\alpha}{\alpha+1}} + C_2\beta^{\alpha} \\ &= C_1^{1 - \frac{1}{\alpha+1}} (\alpha C_2)^{\frac{1}{\alpha+1}} + C_2^{1 - \frac{\alpha}{\alpha+1}} \left(\frac{C_1}{\alpha}\right)^{\frac{\alpha}{\alpha+1}} + C_2\beta^{\alpha} \\ &= C_1^{\frac{\alpha}{\alpha+1}} (\alpha C_2)^{\frac{1}{\alpha+1}} + C_2^{\frac{1}{\alpha+1}} C_1^{\frac{\alpha}{\alpha+1}} \alpha^{\frac{-\alpha}{\alpha+1}} + C_2\beta^{\alpha} \\ &= (C_1^{\alpha} C_2)^{\frac{1}{\alpha+1}} \left( \alpha^{\frac{1}{\alpha+1}} + \alpha^{\frac{-\alpha}{\alpha+1}} \right) + C_2\beta^{\alpha} \\ &= (C_1^{\alpha} C_2)^{\frac{1}{\alpha+1}} \alpha^{\frac{-\alpha}{\alpha+1}} \left( \alpha^{\frac{1+\alpha}{\alpha+1}} + 1 \right) + C_2\beta^{\alpha} \\ &= (C_1^{\alpha} C_2)^{\frac{1}{\alpha+1}} \alpha^{\frac{-\alpha}{\alpha+1}} (\alpha + 1) + C_2\beta^{\alpha} \end{align*}
Thus, we have the upper bound:
\begin{align}\label{eq:bound_intermediate}
    \min_{\gamma > 0} f(\gamma) \le (\alpha+1) \alpha^{\frac{-\alpha}{\alpha+1}} (C_1^{\alpha} C_2)^{\frac{1}{\alpha+1}} + C_2 \beta^\alpha 
\end{align}
Now, we substitute the definitions of $C_1$, $C_2$, and $\alpha$. Let $\bar{G} = \frac{1}{mn} \sum_{i,j}\mathbb{E}[\|\nabla\ell(A_k(S);Z_{ij})\|^2]$ denote the average expected squared norm of the gradient. Then $C_1 = \bar{G}/2$. Let $H = \frac{e\sqrt{2\beta}c^2T^{2c\beta}}{m}$. Then $C_2 = \frac{2c\beta+2}{n(2c\beta+1)} (4H)^{\alpha}$.

We also need the following exponent relations based on $\alpha = \frac{1}{2c\beta+2}$:
\begin{align*}
    \alpha+1 &= \frac{1}{2c\beta+2} + 1 = \frac{2c\beta+3}{2c\beta+2}\\
    \frac{1}{\alpha+1} &= \frac{2c\beta+2}{2c\beta+3}\\
    \frac{\alpha}{\alpha+1} &= \frac{1/(2c\beta+2)}{(2c\beta+3)/(2c\beta+2)} = \frac{1}{2c\beta+3}\\
    \frac{-\alpha}{\alpha+1} &= -\frac{1}{2c\beta+3}
\end{align*}
Let's evaluate the two terms in the bound \eqref{eq:bound_intermediate}.

\textbf{First Term:} $(\alpha+1) \alpha^{\frac{-\alpha}{\alpha+1}} (C_1^{\alpha} C_2)^{\frac{1}{\alpha+1}}$
\begin{align*} C_1^{\alpha} C_2 &= \left(\frac{\bar{G}}{2}\right)^{\alpha} \frac{2c\beta+2}{n(2c\beta+1)} (4H)^{\alpha} = \frac{2c\beta+2}{n(2c\beta+1)} \left( \frac{\bar{G}}{2} \cdot 4H \right)^{\alpha} \\ &= \frac{2c\beta+2}{n(2c\beta+1)} (2\bar{G}H)^{\alpha} \end{align*}
\begin{align*} (C_1^{\alpha} C_2)^{\frac{1}{\alpha+1}} &= \left( \frac{2c\beta+2}{n(2c\beta+1)} \right)^{\frac{1}{\alpha+1}} (2\bar{G}H)^{\frac{\alpha}{\alpha+1}} \\ &= \left( \frac{2c\beta+2}{n(2c\beta+1)} \right)^{\frac{2c\beta+2}{2c\beta+3}} (2\bar{G}H)^{\frac{1}{2c\beta+3}} \end{align*}
The coefficient is:
$$ (\alpha+1) \alpha^{\frac{-\alpha}{\alpha+1}} = \frac{2c\beta+3}{2c\beta+2} \left(\frac{1}{2c\beta+2}\right)^{-\frac{1}{2c\beta+3}} = \frac{2c\beta+3}{2c\beta+2} (2c\beta+2)^{\frac{1}{2c\beta+3}} $$
Combining these parts for the first term:
\begin{align*} \text{First Term} 
&= \left( \frac{2c\beta+3}{2c\beta+2} (2c\beta+2)^{\frac{1}{2c\beta+3}} \right) \left( \frac{2c\beta+2}{n(2c\beta+1)} \right)^{\frac{2c\beta+2}{2c\beta+3}} (2\bar{G}H)^{\frac{1}{2c\beta+3}} \\ 
&= \frac{2c\beta+3}{2c\beta+2} (2c\beta+2)^{\frac{1}{2c\beta+3}} \frac{(2c\beta+2)^{\frac{2c\beta+2}{2c\beta+3}}}{(n(2c\beta+1))^{\frac{2c\beta+2}{2c\beta+3}}} (2\bar{G}H)^{\frac{1}{2c\beta+3}} \\ 
&= (2c\beta+3) (2c\beta+2)^{-1 + \frac{1}{2c\beta+3} + \frac{2c\beta+2}{2c\beta+3}} \left( n(2c\beta+1) \right)^{-\frac{2c\beta+2}{2c\beta+3}} (2\bar{G}H)^{\frac{1}{2c\beta+3}} \\ &= (2c\beta+3)  \left( n(2c\beta+1) \right)^{-\frac{2c\beta+2}{2c\beta+3}} (2\bar{G}H)^{\frac{1}{2c\beta+3}} \\ 
&= (2c\beta+3) \left( n(2c\beta+1) \right)^{-\frac{2c\beta+2}{2c\beta+3}} \left( \frac{2 \bar{G} e\sqrt{2\beta}c^2T^{2c\beta}}{m} \right)^{\frac{1}{2c\beta+3}} \end{align*}

\textbf{Second Term:} $C_2 \beta^\alpha$
\begin{align*} C_2 \beta^\alpha &= \left( \frac{2c\beta+2}{n(2c\beta+1)} (4H)^{\alpha} \right) \beta^{\alpha} = \frac{2c\beta+2}{n(2c\beta+1)} (4\beta H)^{\alpha} \\ &= \frac{2c\beta+2}{n(2c\beta+1)} \left( \frac{4\beta e\sqrt{2\beta}c^2T^{2c\beta}}{m} \right)^{\frac{1}{2c\beta+2}} \end{align*}

\textbf{Final Upper Bound:}
Combining the two terms, we obtain the final upper bound for the minimum value of $f(\gamma)$:
\begin{align}\label{eq:final_bound}
\min_{\gamma > 0} f(\gamma) &\le \frac{2c\beta+3}{\left(n(2c\beta+1)\right)^{\frac{2c\beta+2}{2c\beta+3}}} \left( \frac{2 \bar{G} e\sqrt{2\beta}c^2T^{2c\beta}}{m} \right)^{\frac{1}{2c\beta+3}} \!\!\!\!\!\!\!+ \frac{2c\beta+2}{n(2c\beta+1)} \left( \frac{4\beta e\sqrt{2\beta}c^2T^{2c\beta}}{m} \right)^{\frac{1}{2c\beta+2}} 
\end{align}
where $\bar{G} = \frac{1}{mn} \sum_{i,j}\mathbb{E}[\|\nabla\ell(A_k(S);Z_{ij})\|^2]$, where $A_k(S) = \theta_k^{(T)}$.

\subsection{Proof of optimization error of DSGD-MGS}

Next, we will analyze the expression $\bar{G} = \frac{1}{mn} \sum_{i,j}\mathbb{E}[\|\nabla\ell(\theta_k^{(T)};Z_{ij})\|^2]$ in detail to further understand the impact of algorithmic parameters in DSGD-MGS on the generalization error bound. Prior to this, since the gradient in equation~(\ref{eq:bound_for_bar_theta}) corresponds to the gradient of the final iteration, and current research in the academic community has not yet thoroughly investigated the gradient of the final iteration in non-convex settings for DSGD, we need to clarify an assumption that is widely used in non-convex optimization. This assumption establishes a connection between the gradient and the function value, enabling us to analyze the specific upper bound of the gradient in equation~(\ref{eq:bound_for_bar_theta}).
\begin{assumption}\label{ass:PL-condition}
    (Polyak-\L{}ojasiewicz Condition) Under the condition that $R_{S_k}(\theta) = \frac{1}{n}\sum_{i=1}^n\ell(\theta;Z_{ik})$ also satisfies the $\beta$-smoothness property, the objective function $R_S(\theta) = \frac{1}{m}\sum_{k=1}^m R_{S_k}(\theta)$ satisfies the Polyak-\L{}ojasiewicz Condition (PLC) with parameter $\mu$, i.e., for all $\forall \theta \in \mathbb{R}^d$.
    \begin{align*}
        \|\nabla R_S(\theta)\|^2\geq2\mu(R_S(\theta)-R_S^*),\quad\mu>0,\quad R_S^*=\min_\theta R_S(\theta).
    \end{align*}
\end{assumption}

Next, we proceed to estimate the upper bound of $\bar{G}$. According to the Bounded Stochastic Gradient Noise assumption 
(Assumption~\ref{ass:bound SG}) and the Bounded Stochastic Gradient Noise assumption (Assumption~\ref{ass:bound_hetero}), the following inequality holds:
\begin{align*}
    \bar{G} &= \frac{1}{mn} \sum_{i,j}\mathbb{E}[\|\nabla\ell(\theta_k^{(T)};Z_{ij})\|^2] = \frac{1}{mn} \sum_{i,j}\mathbb{E}[\|\nabla\ell(\theta_k^{(T)};Z_{ij}) \pm \nabla R_{S_k}(\theta_k^{(T)})\pm \nabla R_S(\theta_k^{(T)})\|^2] \\
    &\leq 3\sigma^2 + 3\xi^2 + 3\mathbb{E}[\|\nabla R_S(\theta_k^{(T)})\|^2]
\end{align*}
Since $\ell$ satisfies the $\beta$-smoothness property, it is straightforward to show that $R_{S}(\theta_k^{(T)})$ also satisfies the $\beta$-smoothness property. Consequently, $R_{S}(\theta)$ also satisfies the self-bounding property in Lemma \ref{le:self-bound}, i.e., $\|\nabla R_{S}(\theta)\| \leq 2\beta R_{S}(\theta)$. Then, we have
\begin{align}\label{eq:gradient_bound}
    \bar{G} \le 3\sigma^2 + 3\xi^2 + 6\beta\mathbb{E}_S[R_S(\theta_k^{(T)})]
\end{align}
Next, we will focus on bounding $\mathbb{E}_S[R_S(\theta)]$.
According to the results from \cite{hashemi2021benefits} [Theorem 1], we have the following lemma:
\begin{lemma}\label{le:optimization error of dsgd_mgs}
    Let  $\Delta^2 := \max_{\mathbf{\theta}^* \in \mathcal{X}^*} \sum_{k=1}^m \|\nabla R_{S_k}(\theta^*)\|^2, R_0 := R_{S}(\theta^{(0)}) - R_S^*$, where $\mathcal{X}^* = \arg\min_{\theta} R_S(\theta)$ and $R_S^* = R_{S}(\widehat{\theta}_{\text{ERM}})$. Suppose Assumptions \ref{ass:smooth} and \ref{ass:PL-condition} hold. Define
    \begin{align*}
    Q_{0} & :=\log{(\bar{\rho}/46)}/\log{\left(1-\frac{\delta\tilde{\gamma}}{2}\right)},\bar{\rho}:=1-\frac{\mu}{m\beta}, \\
    \tilde{\gamma} & =\frac{\delta}{\delta^2 + 8\delta+(4+2\delta)\lambda_{\max}^2(I-W)}.
    \end{align*}
    Then, if the nodes are initialized such that $\theta_k^Q=0$, for any $Q>Q_0$ after $T$ iterations the iterates of DSGD-MGS with $\eta_t = \frac{1}{\beta}$ satisfy
    \begin{align*}
    \mathbb{E}_{S}[R_S(\theta_k^{(T)})]-R_S^*=\mathcal{O}\left(\frac{\Delta^2e^{-\frac{\delta \tilde{\gamma} Q}{4}}} {1-\bar{\rho}} + \left[1+\frac{\beta}{\mu\bar{\rho}}\left(1+e^{-\frac{\delta \tilde{\gamma} Q}{4}}\right)\right]R_0\rho^T\right).
\end{align*}
Here, $\delta$ represents the spectral gap of $W$, and $\rho \triangleq 1 - \delta = |\lambda_2(W)|$, both of which are defined in detail in Definition \ref{def:gossip_matrix}.
\end{lemma}

By combining Equation (\ref{eq:gradient_bound}) with Lemma \ref{lemma:key-non-conv_paper}, we obtain the upper bound for $\bar{G}$.
\begin{align*}
    \bar{G} = \mathcal{O}(\sigma^2 + \xi^2 + R_S^*) + \mathcal{O}\left(\frac{\Delta^2e^{-\frac{\delta \tilde{\gamma} Q}{4}}} {1-\bar{\rho}} + \left[1+\frac{\beta}{\mu\bar{\rho}}\left(1+e^{-\frac{\delta \tilde{\gamma} Q}{4}}\right)\right]R_0\rho^T\right).
\end{align*}

\section{Concensus Error Analysis}\label{sec:appendix_error_consensus}


\begin{lemma}[Consensus Error Recursion for DSGD-MGS]\label{lem:consensus_error_revised}
Consider the DSGD-MGS algorithm (Algorithm \ref{alg:dsgd-mgs}) under Assumptions \ref{ass:smooth} ($\beta$-smoothness, with $\ell(\theta; z) \in [0, 1]$ implying gradient bound via Lemma \ref{le:self-bound}), \ref{ass:bound SG} (bounded stochastic gradient noise $\sigma^2$), and \ref{ass:bound_hetero} (bounded heterogeneity $\delta^2$), using a symmetric doubly stochastic communication matrix $W$ with $\rho = |\lambda_2(W)| < 1$ (Definition \ref{def:gossip_matrix}). Let $x_t = \mathbb{E}[\frac{1}{m} \sum_{k=1}^m \|\theta_k^{(t)} - \bar{\theta}^{(t)}\|^2]$ be the average consensus error at the start of iteration $t$, where $\bar{\theta}^{(t)} = \frac{1}{m}\sum_{k=1}^m \theta_k^{(t)}$. Then, for any iteration $t \ge 0$ and number of gossip steps $Q \ge 1$, the consensus error satisfies the following recursion:
\begin{equation*}
    x_{t+1} \le \rho^{2Q} (2 + 24\beta^2 \eta_t^2) x_t + 24 \rho^{2Q} (\sigma^2 + \delta^2)\eta_t^2
\end{equation*}
\end{lemma}

\begin{proof}
The proof proceeds in three steps. Let $\mathbb{E}[\cdot]$ denote expectation conditional on the history $\mathcal{F}_t$.

\textbf{Step 1: Bounding the consensus error after local updates.}
Let $\theta_k^{(t,0)} = \theta_k^{(t)} - \eta_t g_k^{(t)}$ and $\bar{\theta}^{(t,0)} = \bar{\theta}^{(t)} - \eta_t \bar{g}^{(t)}$. The consensus error after the local update is $x_{t,0} = \mathbb{E}[\frac{1}{m} \sum_{k=1}^m \|\theta_k^{(t,0)} - \bar{\theta}^{(t,0)}\|^2]$. We have $\theta_k^{(t,0)} - \bar{\theta}^{(t,0)} = (\theta_k^{(t)} - \bar{\theta}^{(t)}) - \eta_t (g_k^{(t)} - \bar{g}^{(t)})$.
\begin{align}
    x_{t,0} &= \mathbb{E}\left[\frac{1}{m} \sum_{k=1}^m \|(\theta_k^{(t)} - \bar{\theta}^{(t)}) - \eta_t (g_k^{(t)} - \bar{g}^{(t)})\|^2\right] \nonumber \\
    &\le \mathbb{E}\left[\frac{1}{m} \sum_{k=1}^m \left( 2\|\theta_k^{(t)} - \bar{\theta}^{(t)}\|^2 + 2\eta_t^2 \|g_k^{(t)} - \bar{g}^{(t)}\|^2 \right)\right] \nonumber \\
    &= 2 x_t + 2\eta_t^2 \mathbb{E}\left[\frac{1}{m} \sum_{k=1}^m \|g_k^{(t)} - \bar{g}^{(t)}\|^2\right]. \label{eq:xt0_bound_step1_revised}
\end{align}
where the inequality follows from $\|a - b\|^2 \le 2\|a\|^2 + 2\|b\|^2$.
Next, we bound the gradient difference term. Let $c = \nabla R_S(\bar{\theta}^{(t)})$. Using $\|x-y\|^2 \le 2\|x-z\|^2 + 2\|y-z\|^2$ and Jensen's inequality:
\begin{align}
    \mathbb{E}\left[\frac{1}{m}\sum_k \|g_k^{(t)} - \bar{g}^{(t)}\|^2\right] &\le \mathbb{E}\left[\frac{1}{m}\sum_k (2\|g_k^{(t)} - c\|^2 + 2\|\bar{g}^{(t)} - c\|^2) \right] \nonumber \\
    &= 2 \mathbb{E}\left[\frac{1}{m}\sum_k \|g_k^{(t)} - c\|^2\right] + 2 \mathbb{E}[\|\bar{g}^{(t)} - c\|^2] \nonumber \\
    &\le 2 \mathbb{E}\left[\frac{1}{m}\sum_k \|g_k^{(t)} - c\|^2\right] + 2 \mathbb{E}\left[\frac{1}{m}\sum_k \|g_k^{(t)} - c\|^2\right] \nonumber \\
    &= 4 \mathbb{E}\left[\frac{1}{m}\sum_k \|g_k^{(t)} - \nabla R_S(\bar{\theta}^{(t)})\|^2\right]. \label{eq:variance_to_mse_bound}
\end{align}
Now, we bound the term $\mathbb{E}[\frac{1}{m}\sum_k \|g_k^{(t)} - \nabla R_S(\bar{\theta}^{(t)})\|^2]$ by decomposing it into three parts using the triangle inequality:
\begin{align}
    &\mathbb{E}\left[\frac{1}{m}\sum_k \|g_k^{(t)} - \nabla R_S(\bar{\theta}^{(t)})\|^2\right] \nonumber \\
    &\le \mathbb{E}\left[\frac{1}{m}\sum_k 3\left( \|g_k^{(t)} - \nabla R_{S_k}(\theta_k^{(t)})\|^2 + \|\nabla R_{S_k}(\theta_k^{(t)}) - \nabla R_S(\theta_k^{(t)})\|^2 \right. \right. \nonumber \\
    &\qquad \qquad \qquad \left. \left. + \|\nabla R_S(\theta_k^{(t)}) - \nabla R_S(\bar{\theta}^{(t)})\|^2 \right) \right] \nonumber \\
    &\le 3(\sigma^2 + \delta^2) + 3\beta^2 x_t. \label{eq:mse_decomposition_bound}
\end{align}
Here, the first inequality uses $\|a+b+c\|^2 \le 3(\|a\|^2 + \|b\|^2 + \|c\|^2)$. The second inequality applies Assumption \ref{ass:bound SG}, Assumption \ref{ass:bound_hetero}, and Assumption \ref{ass:smooth} (for the $\beta$-smoothness of $R_S$, which follows from the smoothness of $\ell$).

Substituting \eqref{eq:mse_decomposition_bound} and (\ref{eq:variance_to_mse_bound}) into \eqref{eq:xt0_bound_step1_revised}:
\begin{equation}
    x_{t,0} \le (2 + 24\beta^2 \eta_t^2) x_t + 24(\sigma^2 + \delta^2)\eta_t^2. \label{eq:xt0_final_bound_revised}
\end{equation}

\textbf{Step 2: Analyzing the effect of $Q$ gossip steps.}
This step analyzes how the consensus error $x_{t,0} = \mathbb{E}[\frac{1}{m} \sum_{k=1}^m \|\theta_k^{(t,0)} - \bar{\theta}^{(t,0)}\|^2]$ evolves during the $Q$ gossip steps defined in Algorithm \ref{alg:dsgd-mgs}, line 9-11, resulting in the state $\theta_k^{(t+1)} = \theta_k^{(t,Q)}$ with consensus error $x_{t+1} = \mathbb{E}[\frac{1}{m} \sum_{k=1}^m \|\theta_k^{(t+1)} - \bar{\theta}^{(t+1)}\|^2]$.

First, we establish that the average model parameter is invariant under the gossip updates because $W$ is doubly stochastic (Definition \ref{def:gossip_matrix}). Let $\bar{\theta}^{(t,q)} = \frac{1}{m}\sum_k \theta_k^{(t,q)}$. Then,
\begin{align*}
    \bar{\theta}^{(t,q+1)} &= \frac{1}{m}\sum_{k=1}^m \theta_k^{(t,q+1)} = \frac{1}{m}\sum_{k=1}^m \sum_{l=1}^m W_{kl} \theta_l^{(t,q)} \\
    &= \frac{1}{m}\sum_{l=1}^m \left( \sum_{k=1}^m W_{kl} \right) \theta_l^{(t,q)}.
\end{align*}
Since $W$ is doubly stochastic, its column sums are equal to 1, i.e., $\sum_{k=1}^m W_{kl} = 1$ for all $l$. Thus,
\begin{align*}
    \bar{\theta}^{(t,q+1)} &= \frac{1}{m}\sum_{l=1}^m (1) \theta_l^{(t,q)} = \bar{\theta}^{(t,q)}.
\end{align*}
By induction, $\bar{\theta}^{(t,Q)} = \bar{\theta}^{(t,Q-1)} = \dots = \bar{\theta}^{(t,0)}$. Therefore, the average model after $Q$ steps is the same as before gossip: $\bar{\theta}^{(t+1)} = \bar{\theta}^{(t,0)}$.

Now, let's analyze the evolution of the deviations from the average. Define the deviation for agent $k$ at gossip step $q$ as $\delta_k^{(t,q)} = \theta_k^{(t,q)} - \bar{\theta}^{(t,0)}$ (note we use the constant average $\bar{\theta}^{(t,0)}$). The initial deviation is $\delta_k^{(t,0)} = \theta_k^{(t,0)} - \bar{\theta}^{(t,0)}$ and the final deviation is $\delta_k^{(t+1)} = \theta_k^{(t+1)} - \bar{\theta}^{(t+1)} = \theta_k^{(t,Q)} - \bar{\theta}^{(t,0)}$.
The update rule for the deviations is:
\begin{align*}
    \delta_k^{(t,q+1)} &= \theta_k^{(t,q+1)} - \bar{\theta}^{(t,0)} = \sum_{l=1}^m W_{kl} \theta_l^{(t,q)} - \bar{\theta}^{(t,0)} \\
    &= \sum_{l=1}^m W_{kl} (\delta_l^{(t,q)} + \bar{\theta}^{(t,0)}) - \bar{\theta}^{(t,0)} \\
    &= \sum_{l=1}^m W_{kl} \delta_l^{(t,q)} + \left(\sum_{l=1}^m W_{kl}\right) \bar{\theta}^{(t,0)} - \bar{\theta}^{(t,0)}.
\end{align*}
Since $W$ is doubly stochastic, its row sums are also 1, i.e., $\sum_{l=1}^m W_{kl} = 1$. Therefore,
\begin{align*}
    \delta_k^{(t,q+1)} = \sum_{l=1}^m W_{kl} \delta_l^{(t,q)}.
\end{align*}
Stacking the deviations into a large vector $\delta^{(t,q)} = [\delta_1^{(t,q)}{}^\top, \dots, \delta_m^{(t,q)}{}^\top]^\top \in \mathbb{R}^{md}$, the update becomes $\delta^{(t,q+1)} = (W \otimes I_d) \delta^{(t,q)}$, where $I_d$ is the $d \times d$ identity matrix and $\otimes$ denotes the Kronecker product. After $Q$ steps, we have:
\begin{equation*}
    \delta^{(t+1)} = (W \otimes I_d)^Q \delta^{(t,0)} = (W^Q \otimes I_d) \delta^{(t,0)}.
\end{equation*}
The consensus error after $Q$ steps is $x_{t+1} = \mathbb{E}[\frac{1}{m} \sum_{k=1}^m \|\delta_k^{(t+1)}\|^2] = \frac{1}{m}\mathbb{E}[\|\delta^{(t+1)}\|^2]$. We bound the squared norm:
\begin{align*}
    \|\delta^{(t+1)}\|^2 &= \|(W^Q \otimes I_d) \delta^{(t,0)}\|^2 \\
    &\le \|W^Q \otimes I_d\|_2^2 \|\delta^{(t,0)}\|^2.
\end{align*}
Using the property of the spectral norm for Kronecker products, $\|A \otimes B\|_2 = \|A\|_2 \|B\|_2$, we have:
\begin{align*}
    \|W^Q \otimes I_d\|_2 = \|W^Q\|_2 \|I_d\|_2 = \|W^Q\|_2.
\end{align*}
Since $\delta^{(t,0)}$ represents deviations from the mean, it holds that $\sum_{k=1}^m \delta_k^{(t,0)} = \mathbf{0}_d$. This means $\delta^{(t,0)}$ lies in the subspace orthogonal to the consensus subspace (vectors of the form $\mathbf{1}_m \otimes v$ for $v \in \mathbb{R}^d$). Let $J = \frac{1}{m}\mathbf{1}\mathbf{1}^T$ be the projection onto the consensus subspace in $\mathbb{R}^m$. The action of $W^Q$ on vectors orthogonal to $\mathbf{1}_m$ is equivalent to the action of $(W-J)^Q$. Therefore, when acting on $\delta^{(t,0)}$, the operator $W^Q \otimes I_d$ acts identically to $(W-J)^Q \otimes I_d$.
The spectral norm $\|(W-J)^Q\|_2$ corresponds to the largest magnitude eigenvalue of $(W-J)^Q$ acting on the orthogonal subspace. Since $W$ is symmetric, the eigenvalues of $W-J$ are $0$ (corresponding to eigenvector $\mathbf{1}$) and $\lambda_i(W)$ for $i=2, \dots, m$. The eigenvalues of $(W-J)^Q$ are $0$ and $\lambda_i(W)^Q$ for $i=2, \dots, m$. Thus,
\begin{align*}
    \|(W-J)^Q\|_2 = \max_{i=2,\dots,m} |\lambda_i(W)^Q| = \left( \max_{i=2,\dots,m} |\lambda_i(W)| \right)^Q = \rho^Q.
\end{align*}
where $\rho = |\lambda_2(W)|$ by Definition \ref{def:gossip_matrix}.
Therefore, $\|W^Q\|_2$ restricted to the relevant subspace is $\rho^Q$. It follows that:
\begin{align*}
    \|\delta^{(t+1)}\|^2 \le (\rho^Q)^2 \|\delta^{(t,0)}\|^2 = \rho^{2Q} \|\delta^{(t,0)}\|^2.
\end{align*}
Taking the expectation and dividing by $m$:
\begin{align}
    x_{t+1} = \frac{1}{m}\mathbb{E}[\|\delta^{(t+1)}\|^2] \le \frac{1}{m}\mathbb{E}[\rho^{2Q} \|\delta^{(t,0)}\|^2] = \rho^{2Q} \left( \frac{1}{m}\mathbb{E}[\|\delta^{(t,0)}\|^2] \right) = \rho^{2Q} x_{t,0}. \label{eq:xt1_xt0_bound_detailed}
\end{align}
This concludes the analysis of the gossip steps.

\textbf{Step 3: Combining the results.}
Substituting the bound for $x_{t,0}$ from \eqref{eq:xt0_final_bound_revised} into \eqref{eq:xt1_xt0_bound_detailed} yields the final result:
\begin{align*}
    x_{t+1} &\le \rho^{2Q} \left[ (2 + 24\beta^2 \eta_t^2) x_t + 24(\sigma^2 + \delta^2)\eta_t^2 \right] \\
    &= \rho^{2Q} (2 + 24\beta^2 \eta_t^2) x_t + 24 \rho^{2Q} (\sigma^2 + \delta^2)\eta_t^2.
\end{align*}
This concludes the proof.
\end{proof}

\begin{remark}[Implications of Lemma \ref{lem:consensus_error_revised}]\label{rem:consensus_implications_revised}
Lemma \ref{lem:consensus_error_revised} establishes a recursive bound for the average consensus error $x_t = \mathbb{E}[\frac{1}{m} \sum_k \|\theta_k^{(t)} - \bar{\theta}^{(t)}\|^2]$. This inequality leads to several key insights regarding the behavior of DSGD-MGS (Algorithm \ref{alg:dsgd-mgs}):
\begin{enumerate}
    \item \textbf{Exponential Error Reduction via MGS:} The recursion $x_{t+1} \le C_t x_t + D_t$ involves coefficients $C_t = \rho^{2Q} (2 + 24\beta^2 \eta_t^2)$ and $D_t = 24 \rho^{2Q} (\sigma^2 + \delta^2)\eta_t^2$. Both coefficients are scaled by $\rho^{2Q}$. Since $\rho = |\lambda_2(W)| < 1$ (Definition \ref{def:gossip_matrix}), increasing the number of gossip steps $Q$ causes $\rho^{2Q}$ to decrease exponentially. Consequently, the influence of past consensus error ($x_t$) and the injection of new error per iteration ($D_t$) are exponentially suppressed as $Q$ increases.

    \item \textbf{Sources of Disagreement:} The term $D_t = 24 \rho^{2Q} (\sigma^2 + \delta^2)\eta_t^2$ arises from the local updates. It explicitly depends on the variance of stochastic gradients ($\sigma^2$, Assumption \ref{ass:bound SG}) and the variance due to data heterogeneity across agents ($\delta^2$, Assumption \ref{ass:bound_hetero}). Multiple gossip steps mitigate the impact of these factors by the exponential factor $\rho^{2Q}$.

    \item \textbf{Convergence of Consensus Error:} The asymptotic behavior of $x_t$ depends on the step size $\eta_t$:
        \begin{itemize}
            \item \emph{Decreasing step size:} If $\{\eta_t\}$ satisfies $\sum_{t=0}^\infty \eta_t = \infty$ and $\sum_{t=0}^\infty \eta_t^2 < \infty$, and if the network connectivity and $Q$ are sufficient such that $2\rho^{2Q} < 1$ (i.e., $\rho^Q < 1/\sqrt{2}$), then the contraction factor $C_t \approx 2\rho^{2Q} < 1$ for large $t$. Since the noise term $D_t$ is proportional to $\eta_t^2$, we have $\sum_{t=0}^\infty D_t < \infty$. Under these conditions, standard results for stochastic approximation (e.g., Robbins-Siegmund lemma) imply that $x_t \to 0$ as $t \to \infty$. The models across agents asymptotically reach consensus.
            \item \emph{Constant step size:} If $\eta_t = \eta$ is constant, convergence to a steady state requires the contraction factor $C = \rho^{2Q}(2 + 24\beta^2 \eta^2)$ to be strictly less than 1. This stability condition, $C < 1$, again necessitates $\rho^Q < 1/\sqrt{2}$ and potentially a small enough step size $\eta$. If $C < 1$, iterating the recursion $x_{t+1} \le C x_t + D$ (where $D = 24 \rho^{2Q} (\sigma^2 + \delta^2)\eta^2$) leads to $\limsup_{t\to\infty} x_t \le \frac{D}{1-C} = \frac{24 \rho^{2Q} (\sigma^2 + \delta^2)\eta^2}{1 - \rho^{2Q} (2 + 24\beta^2 \eta^2)}$. This residual consensus error bound decreases exponentially as $Q$ increases.
        \end{itemize}

    \item \textbf{Approximation of Mini-batch SGD:} The lemma shows that $x_t$ can be made arbitrarily small by choosing a sufficiently large $Q$. When $x_t \approx 0$, all local models are close to the average, i.e., $\theta_k^{(t)} \approx \bar{\theta}^{(t)}$ for all $k$. The effective gradient used to update the average model $\bar{\theta}^{(t+1)}$ is approximately $\frac{1}{m}\sum_k g_k^{(t)} = \frac{1}{m}\sum_k \nabla \ell(\theta_k^{(t)}; Z_{I_k^t k}) \approx \frac{1}{m}\sum_k \nabla \ell(\bar{\theta}^{(t)}; Z_{I_k^t k})$. This is precisely the stochastic gradient estimate used by Mini-batch SGD with a batch size of $m$. Therefore, increasing $Q$ makes DSGD-MGS behave increasingly like Mini-batch SGD, with the deviation (characterized by $x_t$) decaying exponentially with $Q$.However, this relationship holds only for the iterative updates and not for the final generalization error bound.
\end{enumerate}
\end{remark}

\section{Additional results and discussions}\label{appendix_sec:bar{theta}}

\subsection{On the generalization of \texorpdfstring{$A(S) = \bar{\theta}^{(T)} $}{A(S) = θ̄(T)}}
Our generalization bound also holds for the average of the final iterates \( A(S) = \bar{\theta}^{(T)} \triangleq \frac{1}{m}\sum_{k=1}^m\theta_{k}^{(T)}\). We proceed to prove this result.
\begin{proposition}
Let $A(S) = \bar{\theta}^{(T)}$
. Under the same set of hypotheses, except for the form of the gradient expression, the upper-bounds derived in Equation (\ref{eq:final_bound}) also valid upper-bounds on $|\mathbb{E}_{A,S}[R(A(S)) - R_S(A(S))]|$.
\end{proposition}
\begin{proof}
    By replacing \( A_k \) with \( A \) in the proof of Lemma \ref{lemma:key-non-conv} and using the fact that \( \ell \in [0,1] \), we obtain:
    \begin{equation*}\label{eq:A(S)=theta_bar}
\begin{aligned}
    &|\mathbb{E}_{A,S}[R(A_k(S)) - R_S(A_k(S))]| \\
    &\leq \frac{t_0}{n} + \frac{1}{2mn\gamma}\sum_{i,j}\mathbb{E}[\|\nabla\ell(\bar{\theta}^{(T)};Z_{ij})\|^2]   + \frac{\gamma + \beta}{2mn}\sum_{i,j}\mathbb{E}[\|\frac{1}{m}\sum_{k=1}^m\left(\theta_k^{(T)} - \tilde{\theta}_k^{(T)}(i,j)\right)\|_2^2 \big| \mathcal{E}(i,j)]\\
    &\leq \frac{t_0}{n} + \frac{1}{2mn\gamma}\sum_{i,j}\mathbb{E}[\|\nabla\ell(\bar{\theta}^{(T)};Z_{ij})\|^2] + \frac{1}{m}\sum_{k=1}^m\frac{\gamma + \beta}{2mn}\sum_{i,j}\mathbb{E}[\|\theta_k^{(T)} - \tilde{\theta}_k^{(T)}(i,j)\|_2^2 \big| \mathcal{E}(i,j)] 
\end{aligned}
\end{equation*}
According to equation~(\ref{eq:l2_stability_single}), the upper bound of the third term on the right-hand side is independent of the index \( k \). Moreover, the subsequent estimation of the generalization error bound does not rely on the specific form of the gradient but treats it as a constant. Therefore, following the same derivation as for \( A(S) = \theta_{k}^{(T)} \), we obtain the following inequality.
\begin{align}\label{eq:bound_for_bar_theta}
    &|\mathbb{E}_{A,S}[R(A_k(S)) - R_S(A_k(S))]| \\ \nonumber
    &\leq \frac{2c\beta+3}{\left(n(2c\beta+1)\right)^{\frac{2c\beta+2}{2c\beta+3}}} \left( \frac{2 \bar{G} e\sqrt{2\beta}c^2T^{2c\beta}}{m} \right)^{\frac{1}{2c\beta+3}} \!\!\!\!\!\!\!+ \frac{2c\beta+2}{n(2c\beta+1)} \left( \frac{4\beta e\sqrt{2\beta}c^2T^{2c\beta}}{m} \right)^{\frac{1}{2c\beta+2}} 
\end{align}
where $\bar{G} = \frac{1}{mn} \sum_{i,j}\mathbb{E}[\|\nabla\ell(\bar{\theta}^{(T)};Z_{ij})\|^2]$. This completes the proof.
\end{proof}

\subsection{Theoretical proof extended to the mini-batch setting.
}\label{appendix:proof of minibatch}

To make our theory more general, in this section we extend the previous results by incorporating the mini-batch parameter ($b$). Since most of the proof process remains consistent with the earlier analysis, we present only the key modifications and the final conclusions.

\textbf{First}, we modify Assumption \ref{ass:bound SG} to incorporate the mini-batch parameter ($b$). This assumption is quite intuitive, since as the batch size increases, the variance of each gradient estimate decreases.
\begin{assumption}
    (Bounded Stochastic Gradient Noise with mini-batchsize $b$)
    There exists $\sigma^2>0$ such that $\mathbb{E}||\frac{1}{b}\sum_{i=1}^b\nabla\ell(\theta;Z_{i,j})-\nabla R_{\mathcal{S}_j}(\theta)||^2\leq\frac{\sigma^2}{b}$, for any agent $j \in [m]$ and $\theta \in \mathbb{R}^d$.
\end{assumption}

\textbf{Secondly}, Algorithm line 6 updated to mini-batch gradient:
\begin{align*}
    \theta_k^{(t,0)} = \theta_k^{(t)} - \eta_t\frac{1}{b}\sum_{i=1}^{b}\nabla \ell(\theta_k^{(t)};Z_{ik})
\end{align*}

\textbf{Thirdly}, in the probability calculation below Equation \ref{eq:nonconvex_lemma_final}, modify it to:
\begin{align*}
    \mathbb{P}(\mathcal{E}(i,j)^c) \leq \mathbb{P}(T_0 \leq t_0) = \sum_{t=1}^{t_0}\mathbb{P}(T_0=t)\leq \sum_{t=1}^{t_0}\frac{b}{n} = \frac{bt_0}{n}
\end{align*}

With these modifications, we obtain the following generalization bound for decentralized mini-batch SGD with batch size $b$:
\begin{equation}\label{app_eq:mini_batch}
    \begin{aligned}
    &\left|\mathbb{E}_{A,S}[R(A_k(S)) - R_S(A_k(S))]\right| \\
    & \le \frac{(2c\beta+3)b^{\frac{2c\beta+2}{2c\beta+3}}}{\left(n(2c\beta+1)\right)^{\frac{2c\beta+2}{2c\beta+3}}} \left( \frac{2 \bar{G} e\sqrt{2\beta}c^2T^{2c\beta}}{m} \right)^{\frac{1}{2c\beta+3}}  + \frac{b(2c\beta+2)}{n(2c\beta+1)} \left( \frac{4\beta e\sqrt{2\beta}c^2T^{2c\beta}}{m} \right)^{\frac{1}{2c\beta+2}} 
\end{aligned}
\end{equation}

It is important to note that the term $\bar{G}$ includes a variance-related component of order $\mathcal{O}(\sigma^2 / b)$. Combining this with Equation (\ref{app_eq:mini_batch}) and comparing to the single-sample case, we observe that increasing the batch size $b$ actually increases the generalization bound. This implies that larger batches degrade the generalization ability of the algorithm.

From a stability perspective, when drawing a single sample, the probability of selecting the perturbed sample is $\frac{1}{n}$, whereas for batch size $b$, it increases to $\frac{b}{n}$. This leads to earlier and larger accumulation of deviation in the stability term $\delta_k^{(t)}(i,j) = \|\theta_k^{(t)} - \tilde{\theta}_k^{(t)}(i,j)\|_2^2$, and ultimately results in a looser stability and generalization bound.

In addition, the work \cite{keskar2016large} provides a complementary explanation: large batch sizes reduce gradient noise, which increases the likelihood of convergence to sharp minima, known to have poor generalization. In contrast, smaller batches introduce more noise, which helps the model find flatter minima with better generalization. This empirical observation aligns well with our theoretical findings.

\subsection{Experimental validation of the relationship between (b) and (Q)
}\label{appendix:exp_of_b_and_Q}

As shown in our newly introduced theory on the mini-batch parameter ($b$), the batch size represents a trade-off: increasing ($b$) stabilizes gradient estimates but may compromise stability in other aspects. To further validate this conclusion, we conducted experiments on CIFAR-100 using ResNet-18. The results strongly support our claim regarding the interaction between batch size ($b$) and the number of MGS steps ($Q$). Below are the test accuracies (\%) after 300 communication rounds, which clearly reveal the complex interplay between ($b$) and ($Q$).

\begin{table}[htbp]
  \centering
  \small
  \caption{
    Test accuracy (\%) on CIFAR-100 after 300 communication rounds.
  }
  \label{ta:b-q-interplay}
  \begin{tabular}{ccccc}
    \toprule
    & \multicolumn{4}{c}{Number of MGS steps ($Q$)} \\
    \cmidrule(lr){2-5}
    Mini-batch size ($b$) & 1 & 3 & 5 & 10 \\
    \midrule
    16 & 16.08 & 17.98 & 18.75 & \textbf{18.95} \\
    32 & 21.75 & 24.00 & 24.72 & \textbf{24.95} \\
    64 & \textbf{28.38} & 27.21 & 27.80 & 27.39 \\
    96 & 30.39 & \textbf{30.50} & 30.01 & 29.66 \\
    \bottomrule
  \end{tabular}
\end{table}

From Table \ref{ta:b-q-interplay}, we obtain the following key observations:
\begin{itemize}[leftmargin=10pt]
    \item \textbf{Effectiveness of MGS is Conditional on Batch Size:} For smaller batch sizes ($b=16, b=32$), increasing the number of MGS steps ($Q$) consistently and significantly improves performance. For instance, with $b=32$, increasing $Q$ from 1 to 10 boosts accuracy by over 3 percentage points. This aligns with our theory that frequent communication helps mitigate model divergence when local updates are noisy (due to small $b$).
    \item \textbf{Diminishing or Negative Returns of MGS with Large Batches:} Conversely, for larger batch sizes ($b=64$, $b=96$), the benefit of increasing $Q$ diminishes or even becomes negative. With $b=64$, the best performance is achieved with $Q=1$, and further increasing $Q$ harms performance. Similarly, for $b=96$, the peak is at $Q=3$, after which accuracy declines. This suggests that when local gradient estimates are already of high quality (due to large $b$), excessive communication may introduce unnecessary overhead or other negative effects without providing significant consensus benefits.
    \item \textbf{Non-trivial Trade-off and Optimal Configuration:} The results clearly demonstrate that there is no single optimal value for $Q$ that works across all batch sizes. The optimal configuration ($b$, $Q$) is a result of a complex trade-off. For instance, the overall best performance in this early stage of training is achieved at $b=96, Q=3$, not at the highest $Q$ or largest $b$. This empirically validates our argument that local computation and communication are not independent in practice but are linked through a resource and performance trade-off.
\end{itemize}

Overall, these experiments reveal that the optimal configuration of $Q$ and $b$ is the result of a complex trade-off. From this, we can derive empirical guidelines that balance communication efficiency with model performance:

\begin{itemize}[leftmargin=10pt]
    \item \textbf{When the batch size ($b$) is small (e.g., $b=16, 32$):} In this regime, local gradient updates are subject to significant stochasticity (i.e., high gradient noise). Under these conditions, increasing the number of MGS steps ($Q$) yields consistent and substantial performance gains. For instance, raising $Q$ from 1 to 10 effectively promotes model consensus across nodes, mitigating the model divergence caused by gradient noise and thereby enhancing final generalization. This suggests that in scenarios with limited computational resources or where rapid iterations are desired, investing in a moderate increase in communication overhead is highly beneficial.
    \item \textbf{When the batch size ($b$) is large (e.g., $b=64, 96$):} In this case, local gradient estimates are already more accurate, and the impact of gradient noise is reduced. Consequently, the benefits of increasing $Q$ diminish or can even become detrimental. Our results show that the optimal $Q$ is small ($Q=1$ or $Q=3$) in this setting. A possible explanation is that when local updates are of high quality, the marginal gains from intensive communication (high $Q$) do not outweigh the associated communication costs and potential synchronization overhead. It might even disrupt well-trained local features. Therefore, in scenarios where computational power is ample enough to support large-batch training, priority should be given to ensuring sufficient local computation, complemented by a more economical communication strategy.
\end{itemize}

In summary, this experiment provides valuable insights for hyperparameter selection in practical applications: $b$ and $Q$ are not independently tunable but must be co-designed based on available computational and communication resources to strike the optimal balance between performance and cost.

\subsection{Appendix X: On the Technical Necessity and Role of the Polyak-Łojasiewicz Condition}\label{appendix:PL-condition}

In this section, we provide a detailed discussion on the technical role of the Polyak-Łojasiewicz (PL) condition within our generalization analysis. We elucidate why this assumption is instrumental for bounding the final iterate's gradient in the complex setting of non-convex decentralized optimization with Multiple Gossip Steps (MGS), and how it enables the derivation of our main results.

\subsubsection{The Core Challenge: Bounding the Final Iterate's Gradient Norm}

Our main generalization bound in Theorem 3 is derived from the stability analysis in Lemma 2. A critical component of this bound is the term $G$, which represents the expected squared norm of the stochastic gradient at the \textbf{final iterate} of the algorithm, averaged over all clients:
$$
G = \frac{1}{mn} \sum_{i,j} \mathbb{E}\left[\|\nabla\ell(\theta_k^{(T)}); Z_{ij})\|^2\right]
$$
To make this bound useful, we further bound $G$ by a term related to the expected squared norm of the full gradient, $\bar{G} = \mathbb{E}\left[\|\nabla R_S(\theta^{(T)})\|^2\right]$ (as shown in Equation 4.1 and the subsequent analysis). Therefore, the tightness and applicability of our final generalization error bound are directly contingent on our ability to establish a rigorous upper bound for the gradient norm of the \textbf{final iterate}, $\theta^{(T)}$.

However, providing such a bound is a notoriously difficult problem in optimization theory, especially under the confluence of three challenging conditions present in our work: (1) a non-convex objective function, (2) a decentralized training paradigm, and (3) the inclusion of the MGS mechanism. In general non-convex optimization, most convergence guarantees are for the minimum gradient norm over all iterations (i.e., $\min_{t \in \{0, \dots, T-1\}} \mathbb{E}[\|\nabla R_S(\theta^{(t)})\|^2]$), as convergence of the final iterate's gradient is a much stronger and harder-to-prove property.

\subsubsection{Limitations of Existing Last-Iterate Convergence Analyses}

The analysis of last-iterate convergence in non-convex decentralized settings is an active and challenging research frontier. While significant progress has been made, existing theoretical frameworks are not directly applicable to our specific setting.

For instance, the seminal work by Yuan et al. \citep{yuan2022revisiting} provides a last-iterate convergence analysis for D-SGD under non-convexity. However, their analysis is tailored to the standard D-SGD algorithm (equivalent to MGS with $Q=1$) and does not account for the accelerated consensus dynamics introduced by multiple gossip steps ($Q>1$). The MGS mechanism fundamentally alters the interplay between local computation and inter-node communication, rendering direct application of their bounds unsuitable. Other contemporary works on last-iterate convergence often provide bounds on the \textbf{function value gap} (i.e., $\mathbb{E}[\ell(\theta^{(T)})] - R_S^*$) rather than the gradient norm. In a general non-convex landscape, a small function value gap does not necessarily imply a small gradient norm, making these results insufficient for our purpose of bounding $\bar{G}$.

\subsubsection{The PL Condition as a Principled Bridge}

To overcome this theoretical impasse, we adopt the Polyak-Łojasiewicz (PL) condition. The PL condition, defined as $\|\nabla R_S(\theta)\|^2 \geq 2\mu(R_S(\theta) - R_S^*)$, establishes a direct relationship between the squared gradient norm and the function value gap. This is not an ad-hoc choice, but rather a standard and widely accepted technique in the optimization literature when a direct analysis of the gradient norm is intractable. For example, Sun et al. \citep{Sun2022Decentralized} also employed the PL condition in their analysis of decentralized learning to derive tighter theoretical bounds.

The strategic advantage of this approach lies in the fact that a tight, MGS-aware upper bound on the function value gap does exist in the literature, as established by the analysis in Hashemi et al. \citep{hashemi2021benefits}. By leveraging the PL condition, we can translate this existing, powerful result on the function value into a rigorous upper bound on the final iterate's gradient norm, $\bar{G}$, which is precisely what our generalization framework requires.

\subsubsection{The Benefit: Enabling Fine-Grained, Interpretable Generalization Bounds}

This technical choice is what enables us to move beyond high-level, generic bounds and derive some of the \textbf{first fine-grained, MGS-aware generalization guarantees}. By connecting the gradient norm to the MGS-sensitive function value gap, our final bounds in Theorem \ref{the:generalization error of dsgd_mgs} and its subsequent remarks explicitly and quantitatively capture the impact of key algorithmic and architectural hyperparameters. These include:
\begin{itemize}[leftmargin = 10pt]
    \item The number of MGS steps ($Q$), showing an exponential reduction in error.
    \item The communication topology, via the spectral properties of the gossip matrix ($\rho$).
    \item The learning rate ($c$) and total number of iterations ($T$).
    \item The number of clients ($m$) and per-client data size ($n$)
\end{itemize}

This level of detail provides concrete, actionable insights for practitioners and stands in sharp contrast to classic stability analyses (e.g., the L2-stability analysis in \citep{lei2020fine}), which typically yield more abstract bounds, such as a high-level $\mathcal{O}(1/T)$ rate for the optimization error, without explicitly showing the influence of network structure or MGS.

\subsubsection{Modularity and Extensibility of Our Framework}

Finally, it is crucial to recognize that the use of the PL condition is a component of our optimization error analysis (Theorem \ref{the:optimization error of dsgd_mgs paper}), not a fundamental limitation of our stability framework itself. Our overall analytical framework is modular.

This modularity implies that our contribution is extensible. Should future research in optimization theory provide a direct, assumption-free upper bound for the final iterate's gradient norm ($\bar{G}$) in the DSGD-MGS setting, that result could be seamlessly "plugged into" our framework. The stability-derived components of our generalization bound would remain valid, and the overall result would be immediately strengthened and generalized. This highlights that while our work relies on the current state-of-the-art in optimization theory, it is also designed to incorporate future advances.

In summary, our adoption of the PL condition is a deliberate and well-justified technical decision that addresses a significant challenge in current theory, enabling us to provide novel, detailed insights into the generalization behavior of MGS.